\documentclass{article}

\PassOptionsToPackage{numbers, compress}{natbib}
\usepackage[final]{neurips_2023}

\usepackage[utf8]{inputenc} % allow utf-8 input
\usepackage[T1]{fontenc}    % use 8-bit T1 fonts
\usepackage{hyperref}       % hyperlinks
\usepackage{url}            % simple URL typesetting
\usepackage{booktabs}       % professional-quality tables
\usepackage{amsfonts}       % blackboard math symbols
\usepackage{nicefrac}       % compact symbols for 1/2, etc.
\usepackage{microtype}      % microtypography
\usepackage{xcolor}         % colors
\usepackage{amsmath,amssymb,amsthm,bm, bbm}
\usepackage{algorithm}
\usepackage{algorithmic}
\usepackage{subfigure}
\usepackage{graphicx}
\usepackage{adjustbox}
\usepackage{multirow}
\usepackage{enumitem}

\newtheorem{lemma}{Lemma}[section]

\newtheorem{assumption}{Assumption}[section]
\newtheorem{definition}{Definition}[section]
\newtheorem{theorem}{Theorem}

\newtheorem{remark}{Remark}[section]

\ifx\assumption\undefined
\newtheorem{assumption}{Assumption}
\fi

\newcommand{\cD}{{\mathcal{D}}}
\newcommand{\cF}{{\mathcal{F}}}

\newcommand{\cX}{{\mathcal{X}}}
\newcommand{\cZ}{{\mathcal{Z}}}
\newcommand{\cA}{{\mathcal{A}}}
\newcommand{\cB}{{\mathcal{B}}}

\newcommand{\cE}{{\mathcal{E}}}
\newcommand{\cG}{{\mathcal{G}}}
\newcommand{\cS}{{\mathcal{S}}}
\newcommand{\cT}{{\mathcal{T}}}
\newcommand{\cO}{\mathcal{O}}

\newcommand{\hcF}{\widehat{\mathcal{F}}}

\newcommand{\rR}{{\mathbb{R}}}
\newcommand{\E}{{\mathbb{E}}}
\newcommand{\Pb}{{\mathbb{P}}}

\newcommand{\reg}{\mathrm{Regret}}

\newcommand{\hf}{\hat{f}}
\newcommand{\hw}{\hat{w}}
\newcommand{\bff}{\bar{f}}
\newcommand{\tcO}{\tilde{\cO}}
\newcommand{\bcF}{\bar{\cF}}

\newcommand{\beps}{\Bar{\epsilon}}

\def\##1\#{\begin{align}#1\end{align}}
\def\$#1\${\begin{align*}#1\end{align*}}

\usepackage{xcolor}         % colors
\newcommand{\op}{\mathrm{op}}

\newcommand{\argmin}{\mathop{\mathrm{argmin}}}
\newcommand{\argmax}{\mathop{\mathrm{argmax}}}
\newcommand{\subopt}{\mathrm{SubOpt}}

\def\CC{\mathrm{CC}}
\def\Var{\text{Var}}

\title{Corruption-Robust Offline Reinforcement Learning with General Function Approximation}

\author{
    Chenlu Ye$^*$\\
    The Hong Kong University of Science and Technology\\
    cyeab@connect.ust.hk\\
    \And
    Rui Yang$^*$\\
    The Hong Kong University of Science and Technology\\
    ryangam@connect.ust.hk\\
    \And
    Quanquan Gu\\
    University of California, Los Angeles\\
    qgu@cs.ucla.edu\\
    \And
    Tong Zhang\\
    The Hong Kong University of Science and Technology\\
    tongzhang@ust.hk
}
\begin{document}

\maketitle
\def\thefootnote{*}\footnotetext{These authors contributed equally to this work.}\def\thefootnote{\arabic{footnote}}

\begin{abstract}
We investigate the problem of corruption robustness in offline reinforcement learning (RL) with general function approximation, where an adversary can corrupt each sample in the offline dataset, and the corruption level $\zeta\geq0$ quantifies the cumulative corruption amount over $n$ episodes and $H$ steps. Our goal is to find a policy that is robust to such corruption and minimizes the suboptimality gap with respect to the optimal policy for the uncorrupted Markov decision processes (MDPs). Drawing inspiration from the uncertainty-weighting technique from the robust online RL setting \citep{he2022nearly,ye2022corruptionrobust}, we design a new uncertainty weight iteration procedure to efficiently compute on batched samples and propose a corruption-robust algorithm for offline RL. Notably, under the assumption of single policy coverage and the knowledge of $\zeta$, our proposed algorithm achieves a suboptimality bound that is worsened by an additive factor of $\mathcal{O}(\zeta (C(\widehat{\mathcal{F}},\mu)n)^{-1})$ due to the corruption. Here $\widehat{\mathcal{F}}$ is the confidence set, and the dataset $\mathcal{Z}_n^H$, and $C(\widehat{\mathcal{F}},\mu)$ is a coefficient that depends on $\widehat{\mathcal{F}}$ and the underlying data distribution $\mu$. When specialized to linear MDPs, the corruption-dependent error term reduces to $\mathcal{O}(\zeta d n^{-1})$ with $d$ being the dimension of the feature map, which matches the existing lower bound for corrupted linear MDPs. This suggests that our analysis is tight in terms of the corruption-dependent term.
\end{abstract}

\section{Introduction}
Offline reinforcement learning (RL) has received tremendous attention recently because it can tackle the limitations of online RL in real-world applications, e.g., healthcare \citep{wang2018supervised} and autonomous driving \citep{pan2017agile}, where collecting online data is risky, expensive and even infeasible. By leveraging a batch of pre-collected datasets, offline RL aims to find the optimal policy that is covered by the dataset without further interaction with the environment. Due to the restriction of the offline dataset, the utilization of pessimism in the face of uncertainty is widespread 
\citep{an2021uncertainty,bai2022pessimistic,yangrorl,sun2022exploit,ghasemipour2022so} and plays a central role in providing theoretical guarantees for efficient learning \citep{jin2021pessimism,rashidinejad2021bridging,uehara2021pessimistic,wei2022model,xie2021bellman,yin2021towards,zanette2021provable,zhong2022pessimistic,xiong2022nearly}. Notably, these theoretical works demonstrate that a single-policy coverage is sufficient to guarantee sample efficiency.

 In this paper, we study offline RL under adversarial corruption and with general function approximation. Adversarial corruption refers to adversarial attacks on the reward functions and transition dynamics on the data at each step before the learner can access the dataset.
The learner only knows the cumulative corruption level and cannot tell whether the corruption occurs at each data point. Our corruption formulation subsumes the model misspecification \citep{jin2020provably} and the fixed fraction of data contamination \citep{zhang2022corruption} as special cases. Various real-world problems are under the threat of adversarial corruption, such as chatbots misled by discriminative or unethical conversations \citep{neff2016talking,zhang2020adaptive}, and autonomous vehicles tricked by hacked navigation instructions or deliberately contaminated traffic signs \citep{eykholt2018robust}. On the other hand, general function approximation (approximating the value function with a nonlinear function class, such as deep neural networks) plays a pivotal role in modern large-scale RL problems, such as large language model \citep{du2023guiding}, robotics \citep{kober2013reinforcement} and medical treatment \citep{liu2017deep}. Recent works have established different frameworks to explore the minimal structure condition for the function class that enables sample efficiency \citep{jiang2017contextual,wang2020reinforcement,du2021bilinear,jin2021bellman,foster2021statistical,chen2022general,zhong2022posterior}. In particular, \citet{wang2020reinforcement} leverage the concept of eluder dimension \citep{russo2013eluder} and construct the least squares value iteration (LSVI) framework, which establishes optimism at each step for online RL.

For adversarial corruption and general function approximation, a significant amount of research has focused on the online setting. However, offline RL in this setting is still understudied due to restricted coverage conditions and lack of adaptivity. One notable exception is \citet{zhang2022corruption}, which assumes $\epsilon$-fraction of the offline dataset is corrupted, and their algorithm suffers from a suboptimal bound on the corruption term. Our work moves a step further and achieves corruption robustness under the LSVI framework \citep{wang2020reinforcement} in offline RL with general function approximation by generalizing the uncertainty weighting technique \citep{he2022nearly,ye2022corruptionrobust}. We also propose an algorithm robust to an additional known distribution shift. Due to space limit, we defer it to Appendix \ref{s:ds}. We summarize our contributions as follows:
\begin{itemize}[leftmargin=*]
    \item We formally define the corruption level in offline RL. With knowledge of the corruption level, we design an algorithm that draws inspiration from the PEVI algorithm \citep{jin2021pessimism} and the uncertainty-weighting technique. The uncertainty for each data point, serving as the bonus function, is quantified by its informativeness with respect to the whole dataset. We propose the uncertainty weight iteration algorithm to calculate the weights efficiently and prove that the output, an approximation of the uncertainty, is sufficient to control the corruption term.
    \item Theoretically, our proposed algorithm enjoys a suboptimality bound of $\tcO(H(\ln N)^{1/2}(C(\hcF,\mu)n)^{-1/2} + \zeta(C(\hcF,\mu)n)^{-1})$, where $H$ is the episode length, $n$ is the number of episodes, $N$ is the covering number and $C(\hcF,\mu)$ is the coefficient depicting how well the underlying data distribution $\mu$ explores the feature space, and $\hcF$ is the confidence set. The corruption-dependent term reduces to $\cO(\zeta dn^{-1})$ in the linear model of dimension $d$, thus matching the lower bound for corrupted linear MDPs. It is worth highlighting that our novel analysis enables us to eliminate the uncertainty-related weights from the coverage condition.
    \item Motivated by our theoretical findings, we present a practical offline RL algorithm with uncertainty weighting and demonstrate its efficacy under diverse data corruption scenarios. Our practical implementation achieves a $104\%$ improvement over the previous state-of-the-art uncertainty-based offline RL algorithm under data corruption, demonstrating its potential for effective deployment in real-world applications.
\end{itemize}

\subsection{Related Work}

\noindent \textbf{Corruption-Robust Bandits and RL.}
There is an emerging body of theoretical literature on bandits and online RL with corruption. The adversarial corruption is first formulated in the multi-armed bandit problem by \citet{lykouris2018stochastic}, where an adversary corrupts the reward in each round $t$ by $\zeta_t$ and the corruption level is measured by $\zeta=\sum_{t=1}^T|\zeta_t|$. Then, a lower bound with a linear dependence on $\zeta$ is constructed by \citet{gupta2019better}, indicating that the ideal regret bound should achieve a ``parallel'' relationship: $\reg(T) = o(T) + \cO(\zeta)$, and the corruption-independent term approximates the non-corrupted bound. When extending to linear contextual bandits, a line of work \citep{bogunovic2021stochastic,ding2022robust,foster2020adapting,lee2021achieving,zhao2021linear,kang2023robust} propose various methods but either derive sub-optimal regrets or require particular assumptions. The gap is later closed by \citet{he2022nearly}, which achieves the minimax lower bound using a novel sample-dependent weighting technique. Specifically, the weight for each sample is adaptive to its confidence, which is also called uncertainty. Beyond bandits, earlier works on MDPs \citep{chen2021finding, jin2020learning, jin2020simultaneously,luo2021policy,neu2010online, rosenberg2019online, rosenberg2020stochastic} consider the setting where only the rewards are corrupted, and the transitions remain intact. \citet{wu2021reinforcement} begin to handle corruption on both rewards and transitions for tabular MDPs. \citet{wei2022model} establish a unified framework for RL with unknown corruption under a weak adversary, where the corruption happens before the decision is made in each round. Later, \citet{ye2022corruptionrobust} extend the weighting technique \citep{he2022nearly} to corrupted RL with general function approximation and achieve a linear dependence on the cumulative corruption level $\zeta$. Particularly, \citet{wei2022model,ye2022corruptionrobust} both impose corruption on the Bellman operator, which is the same as the corruption model considered in this paper.

\noindent \textbf{Offline RL Against Attacks.} 
The emergence of poisoning attacks in real-world scenarios poses new challenges for offline RL and necessitates improved defenses \cite{wucopa}. There are generally two types of attacks \cite{behzadan2018mitigation}, namely test-time attacks and training-time attacks. In test-time attacks, the training data is clean, and the learned policy must contend with an attacker during test time. For example, \citet{yangrorl} propose learning conservative and smooth policies robust to different test-time attacks. In contrast, our paper focuses on the training-time attack as another line of work \cite{ma2019policy,wucopa,zhang2022corruption}, where part of the training data is corrupted maliciously. \citet{wucopa} propose two certification criteria and a new aggregation-based method to improve the learned policy from corrupted data. To the best of our knowledge, \cite{zhang2022corruption} is the only theoretical work on corrupted offline RL, which considers that an $\epsilon$-fraction ($\epsilon=\zeta/nH$) of samples are corrupted on both rewards and transitions for linear MDPs and achieves an $\cO(\zeta^{1/2}dn^{-1/2})$ suboptimality bound. Notably, distinct from the setting in \citet{zhang2022corruption} that clean data is first collected and then corrupted by an adversary, we consider the setting that data collection and corruption occur at the same time (thus corruption at one step affects the subsequent trajectory). Therefore, our setting is different from that of \citet{zhang2022corruption}. 
%Their proposed algorithm achieves an $\cO(\zeta^{1/2}dn^{-1/2})$ bound on the corruption-related suboptimality, in contrast to our $\cO(\zeta d n^{-1})$ bound when specialized to linear models. In this way, the model of \citet{zhang2022corruption} is encompassed by our framework, which studies offline RL under general function approximation and with cumulative corruption level.

% \citet{zhang2022corruption} theoretically prove that offline RL with corruption is more challenging than the online setting and propose a Robust LSVI algorithm that achieves an optimality gap in the order of $\mathcal{O}(d^{3/2}/\sqrt{N})$ for linear MDPs. Differently, we provide theoretical results for corruption-robust offline RL with general function approximation and validate our theory with empirical results. 

% Despite the comprehensive investigation of online RL settings, the corrupted offline RL is largely under-explored. To the best of our knowledge, \citet{zhang2022corruption} is the only theoretical work that considers offline linear MDP with a $\epsilon$-fraction of samples suffering from corruption (on rewards and transitions). Distinct from \citet{zhang2022corruption}, we study offline RL under general function approximation and with cumulative corruption level $\zeta=\sum_{i,h=1}^{n,H}|\zeta_i^h|$, where $\zeta_i^h$ is the corruption on the Bellman operator at each round $i$ and step $h$.

\section{Preliminaries}
In this section, we formulate the episodic Markov decision process (MDP) with adversarial corruption and under general (nonlinear) function approximation. Before the formal introduction, we introduce some notations to facilitate our presentation.

\noindent \textbf{Notations.} Let $[n]$ denote the set $\{1,\ldots,n\}$. For spaces $\cX$ and $\cA$ and a function $f:\cX\times\cA\rightarrow \mathbb{R}$, let $f(x)=\max_{a\in\cA}f(x,a)$. Given a semi-definite matrix $M$ and a vector $v$, we define $\|v\|_M=\sqrt{v^\top M v}$. For two positive sequences $\{f(n)\}_{n=1}^\infty$, $\{g(n)\}_{n=1}^\infty$,
let $f(n)=\cO(g(n))$ if there exists a constant $C > 0$ such that $f(n)\le Cg(n)$ for all $n \ge 1$, and $f(n)=\Omega(g(n))$ if there exists a constant $C > 0$ such that $f(n)\ge Cg(n)$ for all $n \ge 1$. We use $\tcO(\cdot)$ to omit polylogarithmic factors. Sometimes we use the shorthand notation $z = (x,a)$.

\subsection{Episodic MDPs}
We consider an episodic MDP $(\cX,\cA,H,\Pb,r)$ with the state space $\cX$, action space $\cA$, episode length $H$, transition
kernel $\Pb = \{\Pb^h\}_{h \in [H]}$, and reward function $r = \{r^h\}_{h \in [H]}$. Suppose that the rewards are bounded: $r^h\ge0$ for any $h\in[H]$, and $\sum_{h=1}^H r^h(x^h,a^h) \leq 1$ almost surely. Given any policy $\pi=\{\pi^h:\cX\rightarrow\cA\}_{h \in [H]}$, we define the Q-value and V-value functions starting from step $h$ as
\begin{equation}
\begin{aligned}
Q^h_{\pi}(x^h,a^h)&=\sum_{h'=h}^H \E_{\pi}\big[r^{h'}(x^{h'},a^{h'})\,|\, x^h,a^h\big],\quad V^h_{\pi}(x^h)&=\sum_{h'=h}^H \E_{\pi}\big[r^{h'}(x^{h'},a^{h'})\,|\,x^h\big].
\end{aligned}
\end{equation}
where the expectation $\E_{\pi}$ is taken with respect to the trajectory under the policy $\pi$. There exists an optimal policy $\pi_*$ and optimal value functions $V_*^h(x):=V^h_{\pi^*}(x)=\sup_{\pi} V_\pi^h(x)$ and $Q_*^h(x,a) := Q_{\pi^*}^h(x,a)= \sup_{\pi} Q_\pi^h(x,a)$ that satisfy the Bellman optimality equation:
\#\label{eqn:bellman_opt}
Q_*^h(x,a) = \E_{r^h,x^{h+1}}\big[r^h(s,a) + \max_{a'\in\cA}Q_*^{h+1}(x^{h+1},a') \,|\, x,a\big] := (\cT^h Q_*^{h+1})(x,a),
\#
where $\cT^h$ is called the Bellman operator. Then we define the Bellman residual as
\#
\cE^h(f,x^h,a^h)=f^h(x^h,a^h)-(\cT^h f^{h+1})(x^h,a^h).
\#

\subsection{General Function Approximation}
We approximate the Q-value functions by a function class $\cF = \cF_1 \times \cdots \times\cF_H$ where $\cF_h:\cX\times\cA\rightarrow[0,1]$ for $h \in [H]$, and $f_{H+1} \equiv 0$ since no reward is generated at step $H+1$. Generally, the following assumption is common for the approximation function class.
\begin{assumption}[Realizability and Completeness]\label{as:realizability and completeness}
For all $h\in[H]$, $Q_*^h\in\cF^h$. Additionally, for all $g^{h+1}(x^{h+1})\in[0,1]$, $(\cT^h g^{h+1})(x^h,a^h)\in\cF^h$. 
\end{assumption}
The realizability assumption \citep{jin2021bellman} ensures the possibility of learning the true Q-value function by considering the function class $\cF$. The Bellman completeness (adopted from \citet{wang2020reinforcement} and Assumption 18.22 of \citet{TZ23-lt}) is stronger than that in \citet{jin2021bellman}. The former applies the least squares value iteration (LSVI) algorithm that establishes optimism at each step, while the latter proposes the GOLF algorithm that only establishes optimism at the first step. We use the standard covering number to depict the scale of the function class $\cF$.
\begin{definition}[$\epsilon$-Covering Number]\label{df:Covering Number}
The $\epsilon$-covering number $N(\epsilon,\cF,\rho)$ of a set $\cF$ under metric $\rho$ is the smallest cardinality of a subset $\cF_0\subseteq\cF$ such that for any $f\in\cF$, there exists a $g\in\cF_0$ satisfying that $\rho(f,g)\le\epsilon$. We say $\cF_0$ is an $(\epsilon,\rho)$ cover of $\cF$.
\end{definition}

\subsection{Offline Data Collection Process}
\noindent \textbf{Offline Clean Data.}
Consider an offline clean dataset with $n$ trajectories $\cD=\{(x_i^h,a_i^h,r_i^h)\}_{i,h=1}^{n,H}$. We assume the dataset $\cD$ is compliant with an MDP $(\cX,\cA,H,\Pb,r)$ with the value functions $Q,V$ and the Bellman operator $\cT$: for any policy $\pi$,
\#\label{eq:compliant}
\Pb\big((\cT^hQ_\pi^{h+1})(x_i^h,a_i^h) &= r'+V_\pi^{h+1}(x') \big| \{(x_j^h,a_j^h)\}_{j\in[i]},\{r_j^h,x_j^{h+1})\}_{j\in[i-1]}, Q_\pi^{h+1}\big)\notag\\
& = \Pb\big(r^h(x^h,a^h)=r',x^{h+1}=x' \big| x^h=x_i^h,a^h=a_i^h\big),
\#
where the realizability and completeness in Assumption \ref{as:realizability and completeness} hold. The compliance assumption \eqref{eq:compliant} is also made in \citet{jin2020learning,zhong2022pessimistic}, which means that $\cD$ remains the Markov property and allows $\cD$ to be collected by an adaptive behavior policy.  The induced distribution of the state-action pair is denoted by $\mu=\{\mu^h\}_{h\in[H]}$. %We consider the case where an adversary corrupts the dataset $\cD$ before it is revealed to the learner such that $\cD$ no longer complies with the Bellman operator $\cT$.

\noindent \textbf{Adversarial Corruption.} 
During the offline dataset collection process, after observing the state-action pair $(x^h,a^h)$ chosen by the data collector, an adversary corrupts $r^h$ and $x^{h+1}$ at each step $h$ before they are revealed to the collector. For each corrupted trajectory $i\in[n]$, we define the corrupted value function $Q_i,~V_i$, and the Bellman operator $\cT_i$ satisfying \eqref{eq:compliant}. To measure the corruption level, we notice that characterizing the specific modification on each tuple $(s,a,s',r)$ is hard and unnecessary since once one modifies a tuple, the subsequent trajectory changes. Therefore, it is difficult to tell whether the change is caused by the corruption at the current step or a previous step. In fact, we only care about the part of the change that violates the Bellman completeness. Therefore, following the online setting \citep{ye2022corruptionrobust,wei2022model}, we measure the corruption level by the gap between $\{\cT_i\}_{i=1}^n$ and $\cT$ as follows.
%We consider the corruption setting, where the dataset $\cD$ observed by the learner complies with the corrupted value function $Q_\cD,~V_\cD$, and the Bellman operator $\cT_\cD$ that satisfies \eqref{eq:compliant}. Due to the corruption to the underlying operator $\cT_\cD$, we cannot assume that the considered approximation function class satisfies the completeness in Assumption \ref{as:realizability and completeness}. Instead, we can measure the gap between $\cT_\cD$ and $\cT$ by defining the corruption level.
\begin{definition}[Cumulative Corruption]\label{def:cor_mdp}
The cumulative corruption is $\zeta$ if at any step $h\in[H]$, for a sequence $\{(x_i^h,a_i^h)\}_{i,h=1}^{n,H}\subset\cX\times\cA$ chosen by the data collector and a sequence of functions $\{g^h:\cX\rightarrow[0,1]\}_{h=1}^H$, we have for all $h\in[H]$,
$$ 
\sum_{i=1}^n|\zeta_i^h|\le \zeta^h,\quad \sum_{h=1}^H\zeta^h :=\zeta,
$$
where $\zeta_i^h = (\cT^hg^{h+1}-\cT_i^hg^{h+1})(x_i^h,a_i^h)$.
\end{definition}
%This corruption level $\zeta$ depicts a cumulative upper bound over $H$ steps and $n$ trajectories for the gap between the uncorrupted and corrupted operators $\cT$ and $\cT_{\cD}$. 
The \textbf{learning objective} is to find a policy $\pi$ that minimizes the suboptimality of $\pi$ given any initial state $x^1=x$: $\subopt(\pi, x) = V_*^1(x) - V_{\pi}^1(x),$
where $V(\cdot)$ is the value function induced by the uncorrupted MDP.

\section{Algorithm}
In this section, we first highlight the pivotal role that uncertainty weighting plays in controlling the corruption-related bound. To extend the uncertainty weighting technique to the offline setting, we propose an iteration algorithm. With the proposed algorithm, the theoretical result for the suboptimality is presented.

\subsection{Uncertainty-Related Weights}
In this subsection, we discuss the choice of weight for a simplified model without state transition ($H=1$) and use the notation $z_i=(x_i,a_i)$. Given a dataset $\{(z_i,y_i)\}_{i\in[n]}$, we have $y_i=\bff(z_i)+\zeta_i+\epsilon_i$ for $i\in[n]$, where $\bff\in\cF$ is the uncorrupted true value, the noise $\epsilon_i$ is zero-mean and conditional $\eta$-subGaussian, and the corruption level is $\zeta=\sum_{i=1}^n |\zeta_i|$.

We begin with delineating the consequence caused by the adversarial corruption for the traditional least-square regression:
$
\hf = \min_{f\in\cF} \sum_{i=1}^n\big(f(z_i) - y_i\big)^2.
$
Some calculations lead to the following decomposition:
{\footnotesize
\$
\sum_{i=1}^n (\hf(z_i)-\bff(z_i))^2 = \underbrace{\sum_{i=1}^n\big[(\hf(z
_i)-y_i)^2 - (\bff(z_i)-y_i)^2\big]}_{ I_1\le0} + 2\underbrace{\sum_{i=1}^n(\hf(z_i)-\bff(z_i))\epsilon_i}_{ I_2:\text{Noise~term}} + 2\underbrace{\sum_{i=1}^n(\hf(z_i)-\bff(z_i))\zeta_i}_{ I_3:\text{Corruption~term}}.
\$}
The term $I_1\le 0$ since $\hf$ is the solution to the least-square regression. The term $I_2$ is bounded by $\tcO(\ln N)$ because of the $\eta$-subGaussainity of $\epsilon_i$, where $N$ is the covering number of $\cF$. The term $I_3$ is ruined by corruption: $I_3\le 2\sum_{i=1}^n|\zeta_i|=\cO(\zeta)$. Hence, the confidence radius $(\sum_{i=1}^n (\hf(z_i)-\bff(z_i))^2)^{1/2}=\tcO(\sqrt{\zeta+\ln N})$ will explode whenever the corruption level $\zeta$ grows with $n$.

To control the corruption term, motivated by the uncertainty-weighting technique from online settings~\citep{ye2022corruptionrobust,he2022nearly,zhou2022computationally}, we apply the weighted regression:
$
\hf = \min_{f\in\cF} \sum_{i=1}^n \big(f(z_i) - y_i\big)^2/\sigma_i^2,
$
where ideally, we desire the following uncertainty-related weights:
\#\label{eq:ideal_weight}
\sigma_i^2 = \max\bigg(1,\frac{1}{\alpha}\underbrace{\sup_{f,f'\in\cF}\frac{|f(z_i) - f'(z_i)|}{\sqrt{\lambda + \sum_{j=1}^n(f(z_j) - f'(z_j))^2/\sigma_j^2}}}_{\text{Uncertainty}}\bigg),\quad i=1,\ldots,n,
\#
where $\alpha,\lambda>0$ are pre-determined parameters. The uncertainty quantity in the above equation is the supremum of the ratio between the prediction error $|f(z_i) - f'(z_i)|$ and the training error $\sqrt{\sum_{j=1}^n(f(z_j) - f'(z_j))^2/\sigma_j^2}$ over $f,f'\in\cF$. Intuitively, the quantity depicts the relative information of a sample $z_i$ against the whole training set $\{z_1,\ldots,z_n\}$. We can use the linear function class as a special example to explain it. When the function space $\mathcal F^h$ is embedded into a $d$-dimensional vector space: $\mathcal F^h=\{\langle w(f), \phi(\cdot) \rangle : z\rightarrow\mathcal R\}$, the uncertainty quantity becomes
{\small
\$
\sup_{f,f'\in\mathcal F} \frac{|\langle w(f)-w(f'), \phi(z_i) \rangle|}{\sqrt{\lambda+\sum_{j=1}^n\big(\langle w(f)-w(f'), \phi(z_j) \rangle\big)^2/\sigma_j^2}} &\le \sup_{f,f'\in\mathcal F} \frac{|\langle w(f)-w(f'), \phi(z_i) \rangle|}{\sqrt{\big(w(f)-w(f')\big)^{\top} \Lambda \big(w(f)-w(f')\big)}}\\
&\le \sqrt{\phi^{\top}(z_i)\Lambda^{-1}\phi(z_i)},
\$
}
where $\Lambda=\sum_{j=1}^n\phi(z_j)\phi^{\top}(z_j)/\sigma_j^2$. Moreover, $\big(\phi^{\top}(z_i)\Lambda^{-1}\phi(z_i)\big)^{-1}$ represents the effective number of samples in the $\{z_i\}_{i=1}^n$ along the $\phi(z_i)$'s direction. We discuss in Lemma \ref{lm:equivalence bonus linear and general} that under mild conditions the linear and nonlinear uncertainty quantities are almost equivalent.

However, since the uncertainty
also depends on weights, it is impossible to determine all the weights $\{\sigma_i\}_{i\in[n]}$ simultaneously. Compared with the online setting where the weight in each round can be determined sequentially (iteratively in rounds), we face two challenges in the offline setting: (a) how to compute uncertainty-related weights iteratively? (b) will an approximate solution to the uncertainty play an equivalent role in controlling the corruption term?

\begin{algorithm}[th]
\caption{Uncertainty Weight Iteration}
\label{alg:wi}
\begin{algorithmic}[1]
\STATE {\bf Input:} $\{(x_i, a_i)\}_{i=1}^n,\cF,\alpha>0$
\STATE {\bf Initialization:} $t=0,~\sigma_i^0=1$, $i=1,\ldots,n$
\REPEAT
\STATE $t\leftarrow t+1$
\STATE $(\sigma_i^t)^2 \leftarrow \max\Big(1,\sup_{f,f'\in\cF}\frac{|f(x_i,a_i) - f'(x_i,a_i)|/\alpha}{\sqrt{\lambda + \sum_{j=1}^n(f(x_j,a_j) - f'(x_j,a_j))^2/(\sigma_j^{t-1})^2}}\Big)$, $i=1,\ldots,n$
\UNTIL $\max_{i\in[n]} \big(\sigma_i^{t}/\sigma_i^{t-1}\big)^2 \le 2$
\STATE {\bf Output:} $\{\sigma_i^t\}_{i=1}^n$
\end{algorithmic}
\end{algorithm}

To solve the first challenge, we propose the weight iteration algorithm in Algorithm \ref{alg:wi}. Moreover, we demonstrate the convergence of this algorithm by the monotone convergence theorem in the following lemma, which ensures that the output weights are sufficiently close to desired ones \eqref{eq:ideal_weight}. The proof is provided in Appendix \ref{ss:Proof of Lemma converge_weight}.
\begin{lemma}\label{lm:converge_weight}
There exists a $T$ such that the output of Algorithm \ref{alg:wi} $\{\sigma_i:=\sigma_i^{T+1}\}_{i=1}^n$ satisfy:
\#\label{eq:approximate_weight}
\sigma_i^2 \ge \max\big(1,\psi(z_i)/2\big), \quad \sigma_i^2 \le \max\big(1,\psi(z_i)\big),
\#
where $\psi(z_i)=\sup_{f,f'\in\cF}\frac{|f(z_i) - f'(z_i)|/\alpha}{\sqrt{\lambda + \sum_{j=1}^n(f(z_j) - f'(z_j))^2/\sigma_j^2}}$.
\end{lemma}

For the second challenge, the weighted version $L_n:=\sum_{i=1}^n(\hf(z_i)-\bar f(z_i))^2/\sigma_i^2$ can also be decomposed into three terms correspondingly. We can demonstrate that an approximate choice of weights satisfying \eqref{eq:approximate_weight} is sufficient to control the corruption term as
\$
\sum_{i=1}^n \frac{(\hf(z_i)-\bar f(z_i))\zeta_i}{\sigma_i^2} = \sum_{i=1}^n \frac{|\hf(z_i)-\bar f(z_i)|\zeta_i\cdot\sqrt{\lambda+L_n}}{\sigma_i^2\sqrt{\lambda+\sum_{j=1}^n(\hf(z_j)-\bar f(z_j))^2/\sigma_j^2}} \le 2\alpha\zeta\sqrt{\lambda+L_n}.
\$
Since the corruption-unrelated terms (corresponding to $I_1,I_2$) can still be bounded by $\tcO(\ln N)$, we have $L_n = \tcO(\ln N + 2\alpha\zeta\sqrt{L_n})$, leading to an $\tcO(\alpha\zeta + \sqrt{\ln N})$ confidence radius. Therefore, with a sufficiently small $\alpha$, the effect of corruption can be countered.

\subsection{Corruption-Robust Algorithm}
\begin{algorithm}[th]
\caption{CR-PEVI}
\small
\label{alg:mdp}
\begin{algorithmic}[1]
\STATE {\bf Input:} $\mathcal{D}=\{(x_i^h, a_i^h, r_i^h)\}_{i,h=1}^{n,H},\cF$
\STATE \textbf{Initialization:} Set $f_n^{H+1}(\cdot)\leftarrow0$
\FOR{step $h=H,H-1,\ldots,1$}
\STATE Choosing weights $\{\sigma_i^h\}_{i=1}^n$ by proceeding Algorithm \ref{alg:wi} with inputs $\{(x_i^h,a_i^h)\}_{i=1}^n, \cF^h, \alpha$
\STATE Find the weighted least-squares solution in \eqref{eq:weighted least-square regression}
\STATE Find $\beta^h$ and construct confidence set 
\$
\hcF^h = \Big\{f\in\cF^h: \lambda + \sum_{i=1}^n(f(x_i^h,a_i^h) - \hat f^h(x_i^h,a_i^h))^2/(\sigma_i^h)^2 \le (\beta^h)^2 \Big\}
\$
\STATE Construct bonus function as \eqref{eq:bonus}
\STATE Let $f_n^h(\cdot,\cdot) = \max\big(0,\hat f^h(\cdot,\cdot) - \beta^h b^h(\cdot,\cdot)\big)$
\STATE Set $\hat{\pi}^h(\cdot) = \argmax_{a\in\cA}f_n^h(\cdot, a)$
\ENDFOR
\STATE {\bf Output:} $\{\hat{\pi}^h\}_{h=1}^H$
\end{algorithmic}
\end{algorithm}

Now, for the offline RL with general function approximation, we integrate the uncertainty weight iteration algorithm with the pessimistic value iteration (PEVI) algorithm \citep{jin2021pessimism},  and propose a Corruption-Robust PEVI (CR-PEVI) in Algorithm \ref{alg:mdp}. Our algorithm employs backward induction from step $H$ to 1. Set estimated value function $f_n^{H+1}(\cdot)=0$. At each step $h\in[H]$, having obtained $f_n^{h+1}$, we calculate $f_n^h$ by solving the following weighted least-square regression:
\#\label{eq:weighted least-square regression}
\hat{f}^h = \argmin_{f^h\in\cF^h}\sum_{i=1}^n \frac{(f(x_i^h,a_i^h) - r_i^h - f_n^{h+1}(x_i^{h+1}))^2}{(\sigma_i^h)^2},
\#
where the weights are obtained via Algorithm \ref{alg:wi}.
As opposed to online RL where the necessity of exploration stimulates optimistic estimation, the literature on offline RL \citep{jin2021pessimism,zhong2022pessimistic} is more inclined to pessimism due to the limitation of offline data coverage. Hence, we construct a confidence set $\hcF^h = \{f\in\hcF^h: \lambda + \sum_{i=1}^n(f(x_i^h,a_i^h) - \hat f^h(x_i^h,a_i^h))^2/(\sigma_i^h)^2 \le (\beta^h)^2 \}$ such that the uncorrupted Bellman operator $\cT^h$ converts the value function $f_n^{h+1}$ into the function class $\hcF^h$ (i.e., $\cT^hf_n^{h+1}\in\hcF^h$) with high probability. For the bonus function, we follow \cite{ye2022corruptionrobust} and choose it as
\#\label{eq:bonus}
b_h(x,a) = \sup_{f,f'\in\hcF^h}\frac{|f(x,a) - f'(x,a)|}{\sqrt{\lambda + \sum_{i=1}^n(f(x_i^h,a_i^h) - f'(x_i^h,a_i^h))^2/(\sigma_i^h)^2}},
\#
which is seldom used in practical algorithms due to its unstability. Specifically, the covering number of the space containing \eqref{eq:bonus} may be uncontrollable. According to Appendix E in \cite{ye2022corruptionrobust}, the issue of the covering number can be addressed under mild conditions by some techniques. Therefore, to maintain readability and consistency in this paper, we assume the corresponding bonus function space $\cB^{h+1}$ of \eqref{eq:bonus} has a bounded covering number. Then we introduce pessimism by subtracting $b^h$ from the estimated value function: $f_n^h(x,a)=\max(0,\hf^h(x,a)-\beta^h b^h(x,a))$. %After $H$ steps, CR-PEVI generates the policy $\{\hat{\pi}^h\}_{h=1}^H$.

\section{Theoretical Analysis}
\subsection{Coverage Condition}
No guarantee for the suboptimality can be provided with insufficient data coverage. Based on pessimism, \citet{jin2020learning,rashidinejad2021bridging} have demonstrated that the coverage over the optimal policy is sufficient for sample-efficient offline RL. The following condition covers the optimal policy under general function approximation. 

\begin{definition}[Coverage Coefficient]\label{df:coverage_condition_mdp}
Consider the offline dataset $\{x_i^h,a_i^h\}_{i,h=1}^{n,H}$. For any initial state $x^1=x\in\cX$, the weighted coverage coefficient is:
{\footnotesize
\#\label{eq:coverage coefficient weighted}
\CC^{\sigma}(\lambda,\hcF,\cZ_n^H) = \max_{h\in[H]}\E_{\pi_*}\bigg[\sup_{f,f'\in\hcF^h}\frac{n(f(x^h,a^h) - f'(x^h,a^h))^2/\sigma^h(x^h,a^h)^2}{\lambda + \sum_{i=1}^n(f(x_i^h,a_i^h) - f'(x_i^h,a_i^h))^2/(\sigma_i^h)^2} \,\bigg|\, x^1=x\bigg],
\#}
where $\E_{\pi_*}$ is taken with respect to the trajectory
induced by $\pi_*$ in the underlying uncorrupted MDP, and the weight for the trajectory induced by the optimal policy $\pi_*$ is
{\footnotesize
\#\label{eq:sigma^h(x^h,a^h)^2}
(\sigma^h(x^h,a^h))^2 = \max\bigg(1, \sup_{f,f'\in\hcF^h} \frac{|f(x^h,a^h) - f'(x^h,a^h)|/\alpha}{\sqrt{\lambda + \sum_{i=1}^n(f(x_i^h,a_i^h) - f'(x_i^h,a_i^h))^2/(\sigma_i^h)^2}}\bigg).
\#}
\end{definition}
When the weights $\sigma_i^h$, $\sigma^h(x^h,a^h)$ all equal to $1$, we get the unweighted coverage coefficient
{\footnotesize
\#\label{eq:coverage coefficient unweighted}
\CC(\lambda,\hcF,\cZ_n^H) = \max_{h\in[H]}\E_{\pi_*}\bigg[\sup_{f,f'\in\hcF^h}\frac{n(f(x^h,a^h) - f'(x^h,a^h))^2}{\lambda + \sum_{i=1}^n(f(x_i^h,a_i^h) - f'(x_i^h,a_i^h))^2} \,\bigg|\, x^1=x\bigg].
\#}
In the face of corruption, we require single-policy coverage over the uncorrupted trajectory. This coefficient depicts the expected uncertainty of sample $(x^h,a^h)$ induced by the optimal policy $\pi_*$ compared to the $n$ training samples. We use the linear MDP to interpret this condition, where the function space $\cF^h$ is embedded into a $d$-dimensional vector space: $\cF^h=\{\langle w(f), \phi(\cdot) \rangle : z\rightarrow\rR\}$, where $z$ denotes the state-action pair $(x,a)$. With the notation $\Lambda^h=\lambda I+\sum_{i=1}^n\phi(z_i^h)\phi(z_i^h)^\top$, we can demonstrate that if the sufficient ``coverage'' in \citet{jin2021pessimism} holds: there exists a constant $c^{\dagger}$ such that $\Lambda^h\succeq I+c^{\dagger}n\E_{\pi_*}[\phi(z^h)\phi(z^h)^\top \,|\,x^1=x]$ for all $h\in[H]$, our coverage coefficient is bounded: $\CC(\lambda,\hcF,\cZ_n^H) \le d/c^{\dagger} < \infty$.

Additionally, we introduce a new general version of the well-explored dataset condition, which is the key to eliminating the uncertainty-related weights from the instance-dependent bound and deriving the final result.
\begin{assumption}[Well-Explored Dataset]\label{as:Well-Explored Dataset}
For a function space $\cF$ and data empirical distribution $\mu$, there exists a constant $C(\cF,\mu)>0$ such that for any $h\in[H]$, and two distinct $f,f'\in\cF^h$,
\#\label{eq:condition of lm:Connections between Weighted and Unweighted Coefficient}
\E_{\mu^h}\Big[\big(f(z^h)-f'(z^h)\big)^2\Big] \ge C(\cF,\mu)\|f-f'\|_{\infty}^2.
\#
\end{assumption}
We interpret this condition with the linear model, where the condition \eqref{eq:condition of lm:Connections between Weighted and Unweighted Coefficient} becomes: for any two distinct $f,f'\in\bcF^h$,
\$
\big(w(f)-w(f')\big)^{\top} \E_{z^h\sim\mu^h} [\phi(z^h)\phi(z^h)^\top] \big(w(f)-w(f')\big) \ge C(\hcF,\mu)\|w(f)-w(f')\|_2^2.
\$
As proved in Lemma \ref{lm:minimum eigenvalue condition and condition for W and UW connections}, with high probability, the above condition holds with $C(\hcF,\mu)=\Theta(d^{-1})$ when the $n$ trajectories of $\cD$ are independent, and the data distribution $\mu^h$ satisfies the minimum eigenvalue condition:
\#\label{eq:minimum eigenvalue condition0}
\sigma_{\min}\big(\E_{z^h\sim\mu^h} [\phi(z^h)\phi(z^h)^\top]\big)=\bar{c}/d,
\#
where $\bar c>0$ is an absolute constant. This is a widely-adopted assumption in the literature \citep{duan2020minimax,wang2020statistical,xiong2022nearly}. Note that $\Theta(d^{-1})$ is the largest possible minimum eigenvalue since for any data distribution $\tilde \mu^h$, $\sigma_{\min}(\E_{z^h\sim\mu^h} [\phi(z^h)\phi(z^h)^\top]) \le d^{-1}$ by using $\|\phi(z^h)\|\le 1$ for any $z^h\in\cX\times\cA$.

We will demonstrate in Lemma \ref{lm:Connections between Weighted and Unweighted Coefficient} that the coverage coefficient $CC^{\sigma}$ is controlled by $C(\cF,\mu)$.
%Therefore, the sufficient ``coverage'' condition implies our coverage condition in the linear setting. We will discuss the connection formally in Lemma \ref{lm:Connection between Coverage Coefficients_linear}.

% Additionally, the relationship of the weighted coverage coefficient and unweighted coverage coefficient, notated as $\CC(\lambda,\hcF,\cZ_n^H)$, is shown as follows:
% {\small
% \$
% \CC^{\sigma}(\lambda,\hcF,\cZ_n^H) \le \frac{1}{\alpha\sqrt{\lambda}}\max_{h\in[H]}\E_{\pi_*}\bigg[\sup_{f,f'\in\hcF^h}\frac{n(f(x^h,a^h) - f'(x^h,a^h))^2}{\lambda + \sum_{i=1}^n(f(x_i^h,a_i^h) - f'(x_i^h,a_i^h))^2} \,\bigg|\, x^1=x\bigg] = \frac{\CC(\lambda,\hcF,\cZ_n^H)}{\alpha\sqrt{\lambda}},
% \$}
% where the inequality uses the upper bound of weights: $(\sigma_i^h)^2\le 1/(\alpha\sqrt{\lambda})$.

\subsection{Main Result}
Then, the following theorem ensures that the suboptimality of Algorithm \ref{alg:mdp} has an $\cO(\zeta/n)$ dependence on corruption $\zeta$.
\begin{theorem}\label{th:mdp}
Given corruption $\zeta = \sum_{h=1}^H\zeta^h$ and $\delta>0$, we choose the covering parameter $\gamma=1/(n\max_h\beta^h\zeta^h)$, $\lambda=\ln(N_n(\gamma))$, the weighting parameter $\alpha=H\sqrt{\ln N_n(\gamma)}/\zeta$, and the confidence radius
$$
\beta^h=c_{\beta}\big( \alpha\zeta^h + \sqrt{\ln(HN_n(\gamma)/\delta)}\big),\quad\text{for}~h=H,\ldots,1,
$$
where
$
N_n(\gamma)=\max_h N(\gamma/n,\cF^h)\cdot N(\gamma/n,\cF^{h+1})\cdot N(\gamma/n,\cB^{h+1}(\lambda)).
$
Then, with probability at least $1-2\delta$, the sub-optimality of Algorithm \ref{alg:mdp} is bounded by
\$
\subopt(\hat\pi, x) = \tcO\bigg(3H\sqrt{\frac{\CC^{\sigma}(\lambda,\hcF,\cZ_n^H)\cdot\ln N}{n}} + \frac{\zeta\cdot\CC^{\sigma}(\lambda,\hcF,\cZ_n^H)}{n }\bigg).
\$
Further, if Assumption \ref{as:Well-Explored Dataset} holds,
\$
\subopt(\hat\pi, x) = \tcO\bigg(3H\sqrt{\frac{\ln N}{C(\cF,\mu)n}} + \frac{\zeta}{n C(\cF,\mu)}\bigg).
\$
\end{theorem}
%Our algorithm's performance depends on the corruption level $\zeta$. 
When $\zeta=\cO(\sqrt{n})$, our algorithm achieves the same order of suboptimality as the uncorrupted case. Whenever $\zeta=o(n)$, our algorithm is sample-efficient. Moreover, when specialized to the linear MDP with dimension $d$, where $C(\hcF,\mu)=\Theta(d^{-1})$ and $\ln N_n(\gamma)=\tcO(d^2)$, the suboptimality bound in Theorem \ref{th:mdp} becomes $\tcO(Hd^{3/2}n^{-1/2} + d\zeta n^{-1})$. The corruption-independent term $\tcO(Hd^{3/2}n^{-1/2})$ matches that of PEVI \citep{jin2021pessimism}. The corruption-dependent term nearly matches the lower bound, as will be discussed later.

\begin{remark}
Although the theory requires a known corruption level $\zeta$, in the experiments, we treat the uncertainty ratio $\alpha=O(1/\zeta)$ as a tuning hyperparameter. The use of independent and identically distributed (i.i.d.) trajectories in our experiments renders the hyperparameter tuning process straightforward and conducive to optimizing the performance. Additionally, we can offer a choice of $\alpha=\Theta(1/\sqrt{n})$. This choice finds support in the online setting \citep{ye2022corruptionrobust,he2022nearly}, where this specific choice of $\alpha$ ensures that suboptimality remains in the order of uncorrupted error bound, even when $\zeta=O(\sqrt{n})$.
\end{remark}

%Additionally, this result is instance-dependent and is illustrated by the corrupted linear MDP with dimension $d$, where $\CC^{\sigma}(\lambda,\hcF,\cZ_n^H)=\cO(d /c^{\dagger})$, $\ln N_n(\gamma)=\tcO(d^2)$, and $c^{\dagger}$ is the weighted coverage coefficient in the linear setting. When the size of the dataset $\cD$ is sufficiently large, the majority of samples exhibit low diversity (uncertainty) with respect to the entire offline dataset, i.e., the majority of weights $\sigma_i^h$ equals $1$. Consequently, under the standard (unweighted) coverage condition in \citet{jin2021pessimism}, the suboptimality bound in Theorem \ref{th:mdp} becomes $\tcO(Hd^{3/2}n^{-1/2} + d\zeta n^{-1})$. The corruption-independent term $\tcO(Hd^{3/2}n^{-1/2})$ matches that of PEVI \citep{jin2021pessimism}. The corruption term nearly matches the lower bound shown below. When the diversity of most samples is relatively high, the error bound recovers $\tcO(Hd^{3/2}\zeta^{1/2}n^{-1/2})$.

\noindent \textbf{Proof sketch.} The detailed proof of Theorem \ref{th:mdp} is provided in Appendix \ref{s:Proof of Theorem 1}. Here we present a brief proof sketch for the suboptimality bound, which is accomplished by three steps: (1) by Lemma \ref{lm:Regret_Decomposition_mdp}, if the uncorrupted Bellman backup $\cT^hf_n\in\hcF^h$ for each $h\in[H]$, we can bound the suboptimality by the sum of the bonus $\sum_{h=1}^H\beta^h\E_{\pi_*}[b^h(x^h,a^h)\,|\, x^1=x]$; (2) by Lemma \ref{lm:Confidence_Radius_mdp}, we demonstrate that an approximate uncertainty weight satisfying \eqref{eq:approximate_weight} is the key to bound the weighted Bellman error  $(\sum_{i=1}^n ((\hf^h - (\cT^h f_n^{h+1}))(x_i^h,a_i^h))^2/(\sigma_i^h)^2)^{1/2}$ by $\beta^h=c_{\beta}( \alpha\zeta^h + \sqrt{\ln(HN_n(\gamma)/\delta)})$; and (3) combining the results in the first two steps, we can obtain the suboptimality bounded by:
\$
\subopt(\hat\pi, x) = \tcO\big(H(\CC^{\sigma}(\lambda,\hcF,\cZ_n^H))^{1/2}\cdot(\ln N_n(\gamma))^{1/2}\cdot n^{-1/2} + \zeta\cdot\CC^{\sigma}(\lambda,\hcF,\cZ_n^H)\cdot n^{-1} \big).
\$

To control the weighted coverage coefficient $\CC^{\sigma}$ by  $C(\cF,\mu)$, which is a challenging task due to the intricate form of uncertainty-related weights, we present the following lemma.
\begin{lemma}\label{lm:Connections between Weighted and Unweighted Coefficient}
Under Assumption \ref{as:Well-Explored Dataset} and choose $\beta^h=C_{\beta}\sqrt{\ln N}$ (where $C_{\beta}>0$ contains the logarithmic terms that are omitted) and $\lambda$ given in Theorem \ref{th:mdp}, we have
\$
\CC^{\sigma}(\lambda,\hcF,\cZ_n^H) \le 1/C(\cF,\mu).
\$
\end{lemma}
The main idea of the proof is to use the close relationship between the weights and the uncertainty. See Appendix \ref{s:Connections between Weighted and Unweighted Coefficient} for details.

\noindent \textbf{Lower Bound.} We construct a lower bound for linear MDPs with adversarial corruption $\zeta$ to show that the $\cO(d\zeta n^{-1})$ corruption term of our suboptimality is optimal. The construction of the lower bound is adapted from \citet{zhang2022corruption}, where an $\epsilon$-constant fraction of the dataset is contaminated. We present the proof of Theorem \ref{th:Lower Bound} in Appendix \ref{ss:Lower Bound for Linear MDPs with Corruption}.
\begin{theorem}[Minimax Lower Bound for Linear MDPs]\label{th:Lower Bound}
Under linear MDPs with corruption (Definition \ref{def:cor_mdp}), for any fixed data-collecting distribution $\nu$ satisfying Assumption \ref{as:Well-Explored Dataset}, any algorithm $L:\cD\rightarrow\Pi$ with the knowledge of $\zeta$ cannot find a better policy than $\cO(d\zeta n^{-1})$-optimal policy with probability more than $1/4$:
\$
\min_{L,\nu} \max_{\text{MDP},\cD} \subopt(\hat\pi_L) = \Omega\big(\zeta/nC(\cF,\mu)\big),
\$
where $\cD$ is the corrupted dataset initially generated from the MDP and then corrupted by an adversary, and $\hat\pi_L$ is the policy generated by the algorithm $L$.
\end{theorem}

\section{Experiments}
Based on our theoretical results, we propose a practical implementation for CR-PEVI and verify its effectiveness on simulation tasks with corrupted offline data. 
% More details can be found in Appendix \ref{}.

% \subsection{Practical Implementation}
\noindent \textbf{Practical Implementation.}
To make our algorithm more practical, we use neural networks to estimate the $Q$ function (i.e., $f$ in our theory) and the weight function $\sigma$. 
In linear MDPs, under a sufficiently broad function approximation class $\hcF^h$ and sufficiently small parameter $\lambda$ in Eq. \eqref{eq:bonus}, the bonus function can be simplified as the bootstrapped uncertainty of $Q$ functions, which in turn can be estimated via the standard deviation of an ensemble of $Q$ networks. We defer the detailed discussion to Appendix \ref{ss:Relationship between the Bootstrapped Uncertainty and Bonus}. Following the state-of-the-art  uncertainty-based offline RL algorithm Model Standard-deviation Gradients (MSG) \cite{ghasemipour2022so}, we learn a group of $Q$ networks $Q_{w_i}, i=1,\ldots,K$ with independent targets and optimize a policy $\pi_{\theta}$ with a lower-confidence bound (LCB) objective \cite{ghasemipour2022so,bai2022pessimistic}. Specifically, $Q_{w_i}$ is learned to minimize a weighted regression objective similar to Eq.\eqref{eq:weighted least-square regression}. The weight function $\sigma$ is estimated via bootstrapped uncertainty:
$
    (\sigma(x,a))^2 = \max(\sqrt{\mathbb{V}_{i=1,\ldots,K}\left[Q_{w_i}(x,a)\right]}, 1)
$, where $\mathbb{V}_{i=1,\ldots,K}\left[Q_{w_i}\right]$ is the variance between the group of $Q$ functions. This uncertainty estimation method has also been adopted by prior works \cite{bai2022pessimistic,ghasemipour2022so,yang2023essential}. We refer to our practical algorithm as Uncertainty Weighted MSG (UWMSG) and defer details to Appendix \ref{ap:implement_detail}.

% \subsection{Experimental Setup}
\noindent \textbf{Experimental Setup.}
We assess the performance of our approach using continuous control tasks from \cite{fu2020d4rl} and introduce both random and adversarial attacks on either the rewards or dynamics for the offline datasets. Details about the four types of data corruption and their cumulative corruption levels are deferred to Appendix \ref{ap:implement_detail}. The ensemble size $K$ is set to $10$ for all experiments. For evaluation, we report average returns with standard deviations over 10 random seeds. More implementation details are also provided in Appendix \ref{ap:implement_detail}.

\begin{figure}[t]
    \centering
    \subfigure[]{\includegraphics[width=0.255\linewidth, trim=0 0 3 0, clip]{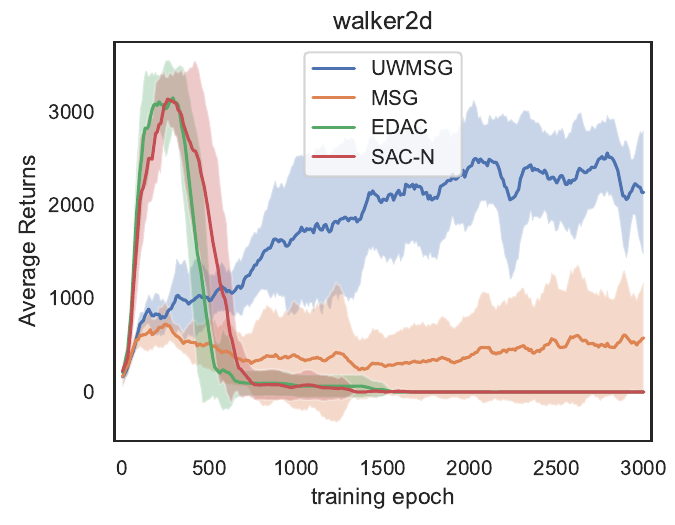}}
    \subfigure[]{\includegraphics[width=0.24\linewidth, trim=19 0 3 0, clip]{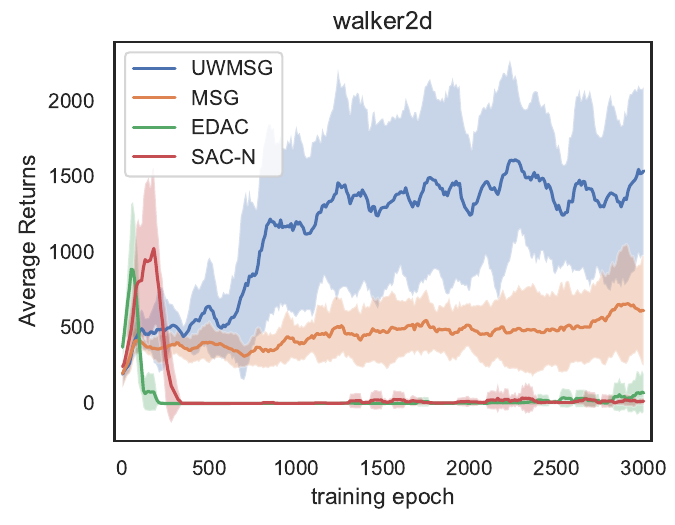}}
     \subfigure[]{\includegraphics[width=0.24\linewidth, trim=19 0 3 0, clip]{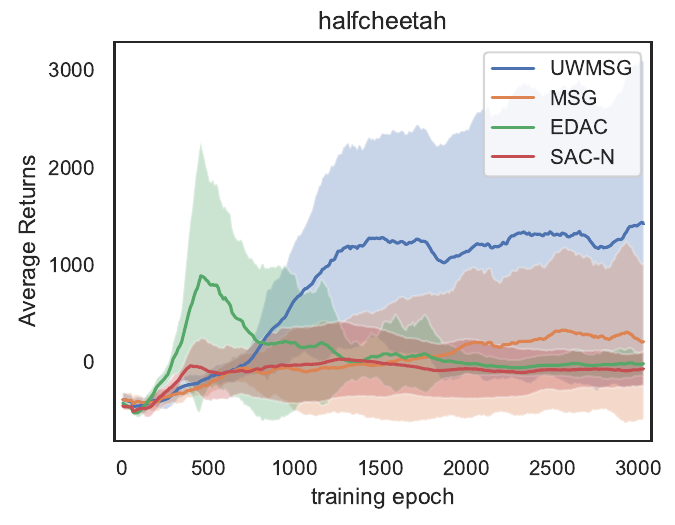}}
     \subfigure[]{\includegraphics[width=0.24\linewidth, trim=19 0 3 0, clip]{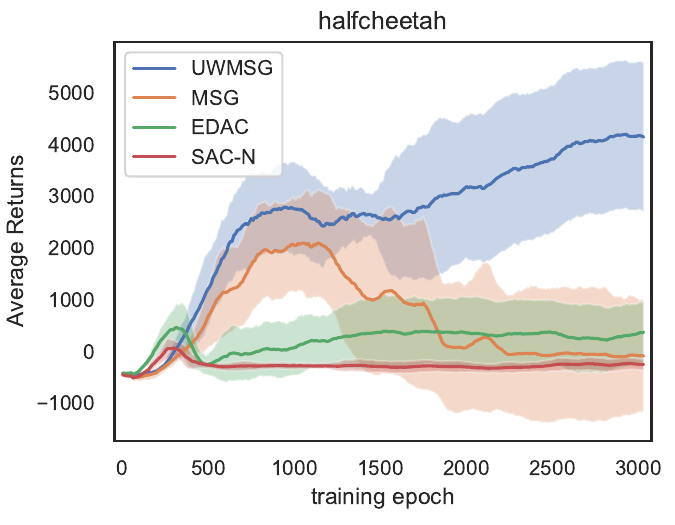}}
    \caption{Performance on the Walker2d and the Halfcheetah tasks under (a) random reward, (b) random dynamics, (c) adversarial reward, and (d) adversarial dynamics attacks.}
    \label{fig:experiments_partial}
\end{figure}

\begin{table}[t]
\small
  \caption{Comparison of different offline RL algorithms under different environments and different data corruption types.}
  \vspace{3pt}
  \label{tab:final_attack}
  \centering
  \begin{adjustbox}{width=0.95\columnwidth}
  \begin{tabular}{l|l|c|c|c|c}
    \toprule
   Environment & Attack Type  & UWMSG & MSG & EDAC & SAC-N \\
    \midrule
   \multirow{4}{*}{Halfcheetah}  &  Random Reward &  7299.9 $\pm$ 169.0 & 4339.5 $\pm$ 3958.7 & 7128.2 $\pm$ 120.5 & \textbf{7357.5 $\pm$ 165.2} \\
    &  Random Dynamics &  \textbf{1425.0 $\pm$ 1659.5}  & 212.9 $\pm$ 793.3 &  -12.3 $\pm$ 131.8 & -66.2 $\pm$ 169.9 \\
    & Adversarial Reward &  \textbf{1016.7 $\pm$ 503.9} & 243.1 $\pm$ 338.6  &  -127.2 $\pm$ 30.3 & -55.7 $\pm$ 24.0 \\
    &  Adversarial Dynamics & \textbf{4144.3 $\pm$ 1437.6} & -87.3 $\pm$ 1055.6 & 374.0 $\pm$ 589.5 & -246.8 $\pm$ 104.9  \\
   \hline
 \multirow{4}{*}{Walker2d}    &  Random Reward & \textbf{2189.7 $\pm$ 603.0} & 539.1 $\pm$ 534.6 & -3.7 $\pm$ 1.2 & -3.8 $\pm$ 1.4 \\
    &  Random Dynamics & \textbf{2278.9 $\pm$ 706.4} & 2122.8 $\pm$ 821.5 &  -3.6 $\pm$ 0.8 & -3.1 $\pm$ 0.7 \\
     &  Adversarial Reward & \textbf{1433.4 $\pm$ 592.1} & 605.8 $\pm$ 310.7 &  61.0 $\pm$ 120.6 & 9.5 $\pm$ 21.6 \\
    &  Adversarial Dynamics & \textbf{946.0 $\pm$ 300.4} & 506.0 $\pm$ 175.5 & -4.8 $\pm$ 0.3 & -5.1 $\pm$ 0.7 \\
    \hline
\multirow{4}{*}{Hopper}    &  Random Reward & \textbf{2021.1 $\pm$ 888.6} & 1599.6 $\pm$ 814.5 & 107.7 $\pm$ 65.5 & 178.8 $\pm$ 114.0 \\
    &  Random Dynamics & \textbf{2116.2 $\pm$ 618.6} & 1552.7 $\pm$ 532.8 &  5.9 $\pm$ 1.2 & 3.9 $\pm$ 2.2 \\
     &  Adversarial Reward & \textbf{751.4 $\pm$ 72.5} & 651.0 $\pm$ 58.0  &  29.2 $\pm$ 18.7 & 111.8 $\pm$ 62.7 \\
    &  Adversarial Dynamics & \textbf{931.2 $\pm$ 227.8} & 717.7 $\pm$ 204.3 & 5.9 $\pm$ 1.2 & 3.9 $\pm$ 2.2 \\
    \hline
    \multicolumn{2}{l|}{Average}   & \textbf{2212.8} & 1083.6 & 630.0 & 607.1 \\
     \bottomrule
  \end{tabular}
  \end{adjustbox}
\end{table}

% \subsection{Experimental Results}
\noindent \textbf{Experimental Results.}
We compare UWMSG with the state-of-the-art uncertainty-base offline RL methods, MSG \cite{ghasemipour2022so}, EDAC \cite{an2021uncertainty}, and SAC-N \cite{an2021uncertainty} under four types of data corruption. In particular, MSG can be considered as UWMSG with a constant weighting function $\sigma(s,a)=1$. As demonstrated in Table \ref{tab:final_attack} and Figure \ref{fig:experiments_partial}, our empirical results find that (1) current offline RL methods are susceptible to data corruption, e.g., MSG, EDAC, SAC-N achieve poor performance under adversarial attacks, and (2) our proposed UWMSG significantly improves performance under different data corruption scenarios, with an average improvement of $104\%$ over MSG. More results can be found in Appendix~\ref{ap:additionl_exp}. In summary, the experimental results validate the theoretical impact of data corruption for value-based offline RL algorithms. Our practical implementation algorithm demonstrates superior efficacy under different data corruption types, thereby highlighting its potential for real-world applications.

\section{Conclusion}
This work investigates the adversarially corrupted offline RL with general function approximation. We propose the uncertainty weight iteration and a weighted version of PEVI. Under a partial coverage condition and a well-explored dataset, our algorithm achieves a suboptimality bound, where the corruption-independent term recovers the uncorrupted bound, and the corruption-related term nearly matches the lower bound in linear models. Furthermore, our experiments demonstrate promising results showing that our practical implementation, UWMSG, significantly enhances the performance of the state-of-the-art offline RL algorithm under reward and dynamics data corruptions.

Our work suggests several potential future directions. First, it remains unsolved whether one can design robust algorithms to handle additional distribution shifts of the state-action pairs. Second, when the corruption level is unknown, how to design a theoretically robust offline RL algorithm and overcome the lack of adaptivity in the offline setting requires further research. Finally, we hope that our work sheds light on applying uncertainty weights to improve robustness in deep offline RL works and even practical applications.

\section*{Acknowledgements}
We thank Wei Xiong and Jiafan He for valuable discussions and feedback on an early draft of this work. We also thank the anonymous reviewers for their helpful comments. %QG is supported in part by the JP Morgan Faculty Research Award.

\bibliography{ref.bib}
\bibliographystyle{apalike}

%%%%%%%%%%%%%%%%%%%%%%%%%%%%%%%%%%%%%%%%%%%%%%%%%%%%%%%%%%%%

\newpage
\appendix
\section{Proof of Theorem \ref{th:mdp}}\label{s:Proof of Theorem 1}
\subsection{Step I: Suboptimality Decomposition}
\begin{lemma}[Regret Decomposition]\label{lm:Regret_Decomposition_mdp}
Assuming that $\cT^hf_n\in\hcF^h$ for all $h\in[H]$, we have
$$
\subopt(\hat{\pi}, x) \le 2\sum_{h=1}^H\beta^h\E_{\pi_*}\big[ b^h(x^h,a^h)\,\big|\, x^1=x\big].
$$
\end{lemma}
\begin{proof}
By invoking Lemma \ref{lm:Suboptimality Decomposition}, we can decompose the suboptimality as follows:
\#\label{eq:suboptimality_decompose}
\subopt(\hat{\pi}, x) &= V_*^1(x) - V_{\hat{\pi}}^1(x)\notag\\
& = \sum_{h=1}^H\E_{\pi_*}\big[f_n^h(x^h,\pi_*(x^h)) - f_n^h(x^h,\hat{\pi}(x^h)) \,\big|\, x^1=x\big]\notag\\
&\qquad - \sum_{h=1}^H\E_{\pi_*}\big[ \cE^h(f_n,x^h,a^h)\,\big|\, x^1=x\big] + \sum_{h=1}^H\E_{\hat{\pi}}\big[ \cE^h(f_n,x^h,a^h)\,\big|\, x^1=x\big]\notag\\
&\le  \underbrace{-\sum_{h=1}^H\E_{\pi_*}\big[ \cE^h(f_n,x^h,a^h)\,\big|\, x^1=x\big]}_{(a)} + \underbrace{\sum_{h=1}^H\E_{\hat{\pi}}\big[ \cE^h(f_n,x^h,a^h)\,\big|\, x^1=x\big]}_{(b)},
\#
where the inequality is due to $\hat\pi(x^h)=\argmax_{a\in\cA} f_n^h(x^h,a)$. Then, we handle the above two terms respectively. 
We first tackle term (b). Supposing that $\cT^h f_n^{h+1} \in \hcF^h$, we get from the definition of the confidence set in Algorithm \ref{alg:mdp} that
\#\label{eq:from the definition of the confidence set}
\Big(\lambda + \sum_{i=1}^n \frac{((\cT^h f_n^{h+1})(x_i^h,a_i^h) - \hat f^h(x_i^h,a_i^h))^2}{(\sigma_i^h)^2}\Big)^{1/2} \le \beta^h.
\#
Therefore, for any $(x^h,a^h)\in\cX\times\cA$, we have
{\small
\$
&\big|\hat{f}^h(x^h,a^h) - (\cT^h f_n^{h+1})(x^h,a^h)\big|\\
&\qquad= \Big(\lambda + \sum_{i=1}^n \frac{((\cT^h f_n^{h+1})(x_i^h,a_i^h) - \hat f^h(x_i^h,a_i^h))^2}{(\sigma_i^h)^2}\Big)^{1/2}\cdot\frac{\big|\hat{f}^h(x^h,a^h) - (\cT^h f_n^{h+1})(x^h,a^h)\big|}{\Big(\lambda + \sum_{i=1}^n ((\cT^h f_n^{h+1})(x_i^h,a_i^h) - \hat f^h(x_i^h,a_i^h))^2/(\sigma_i^h)^2\Big)^{1/2}}\\
&\qquad\le \beta^h b^h(x^h,a^h),
\$}
where the inequality uses \eqref{eq:from the definition of the confidence set} and the definition of the bonus $b^h$ \eqref{eq:bonus}. Then, combining the above result  with $f_n^h(x,a)=\max(0, \hf^h(x,a)-\beta^hb^h(x,a))$ and $(\cT^h f_n^{h+1})(x^h,a^h)\ge0$, we get
$$
- 2\beta^hb^h(x^h,a^h) \le f_n^h(x^h,a^h) - (\cT^h f_n^{h+1})(x^h,a^h) \le 0.
$$
Hence, the term (b) is bounded by
\#\label{eq:suboptimality_term b}
\sum_{h=1}^H\E_{\hat{\pi}}\big[ \cE^h(f_n,x^h,a^h)\,\big|\, x^1=x\big] \le 0.
\#
Moreover, the term (a) is bounded by
\#\label{eq:suboptimality_term a}
 - \sum_{h=1}^H\E_{\pi_*}\big[ \cE^h(f_n,x^h,a^h)\,\big|\, x^1=x\big] \le 2\sum_{h=1}^H\beta^h\E_{\pi_*}\big[ b^h(x^h,a^h)\,\big|\, x^1=x\big].
\#
Finally, by taking \eqref{eq:suboptimality_term a} and \eqref{eq:suboptimality_term b} back into \eqref{eq:suboptimality_decompose}, we conclude the proof.
\end{proof}

\subsection{Step II: Sharper confidence radius for Pessimism}
\begin{lemma}[Confidence Radius]\label{lm:Confidence_Radius_mdp}
In Algorithm \ref{alg:mdp}, for all $h\in[H]$ we have $\cT^hf_n^{h+1}\in\hcF^h$ with probability at least $1-\delta$, where we choose 
$$
\beta^h=24\alpha\zeta^h + \Big(12\lambda + 12\ln(2HN_n^h(\gamma)/\delta) + 12(5\beta^{h+1}\gamma)^2n + 60\beta^{h+1}\gamma\sqrt{nC_1^h(n,\zeta)}\Big)^{1/2},
$$
where 
$$
\begin{aligned}
N_n(\gamma)=\max_h N\Big(\frac{\gamma}{n},\cF^h\Big)\cdot N\Big(\frac{\gamma}{n},\cF^{h+1}\Big)\cdot N\Big(\frac{\gamma}{n},\cB^{h+1}(\lambda)\Big),
\end{aligned}
$$
we use the notation $C_1^h(n,\zeta)=2(\sum_{i=1}^n(\zeta_i^h)^2 + 2n\eta^2 + 3\eta^2\ln(2/\delta))$, and $\zeta^h=\sum_{i=1}^n\zeta_i^h$.
\end{lemma}
\begin{proof}
At each step $h\in[H]$, let $\cF_{\gamma}^{h+1}$ be a $(\gamma,\|\cdot\|_{\infty})$ cover of $\cF^{h+1}$, and $\cB^{h+1}_{\gamma}$ as an $(\gamma,\|\cdot\|_{\infty})$ cover of $\cB^{h+1}(\lambda)$. Then, we construct $\Bar{\cF}_{\gamma}^{h+1}=\cF_{\gamma}^{h+1}\oplus\beta_{\tau}^{h+1}\cB^{h+1}_{\gamma}$ as a $((1+\beta^{h+1})\gamma,\|\cdot\|_{\infty})$ cover of $f_n^{h+1}(\cdot)$. Given $f_n^{h+1}$, we have $\Bar{f}^{h+1}\in\Bar{\cF}_{\gamma}^{h+1}$ so that $\|\Bar{f}^{h+1}-f_n^{h+1}\|_{\infty}\le\beps=(1+\beta^{h+1})\gamma$. Then, we define $y_i^h = r_i^h + f_n^{h+1}(x_i^{h+1})$, $\Bar{y}_i^h=r_i^h+\Bar{f}^{h+1}(x_i^{h+1})$ and
$$
\tilde{f}^h = \argmin_{f^h\in\cF^h}\sum_{i=1}^n (f^h(x_i^h,a_i^h)-\Bar{y}_i^h)^2.
$$
We know from the definition of the covers that
\#\label{eq:approximate_erm}
&\left(\sum_{i=1}^n (\hat{f}^h(x_i^h,a_i^h)-\Bar{y}_i^h)^2\right)^{1/2} \le \left(\sum_{i=1}^n (\hat{f}^h(x_i^h,a_i^h)-y_i^h)^2\right)^{1/2}+\sqrt{n}\beps\nonumber\\
&\qquad \le \left(\sum_{i=1}^n(\tilde{f}^h(x_i^h,a_i^h)-y_i^h)^2\right)^{1/2}+\sqrt{n}\beps \le \left(\sum_{i=1}^n (\tilde{f}^h(x_i^h,a_i^h)-\Bar{y}_i^h)^2\right)^{1/2}+2\sqrt{n}\beps,
\#
where the first and third inequality is due to $\|\Bar{f}^{h+1}-f_n^{h+1}\|_{\infty}\le\beps$, and the second inequality comes from the fact that $\hat{f}^h$ is the ERM solution to the least squares problem. Then, we can invoke Lemma \ref{lm:empirical_diff_mdp} by taking $f_*$ as $\E[\Bar{y}_i^h|x_i^h,a_i^h]=(\cT_i^h\Bar{f}^{h+1})(x_i^h,a_i^h)$ and $f_b$ as $\cT^h\Bar{f}^{h+1}$. With probability at least $1-\delta$, we obtain:
\#\label{eq:sum_hat_fkh_Tb_barf0}
&\sum_{i=1}^n \frac{\left(\hat{f}^h(x_i^h,a_i^h)-(\cT^h\Bar{f}^{h+1})(x_i^h,a_i^h)\right)^2}{(\sigma_i^h)^2}\nonumber\\
&\qquad\le 10\ln(2HN_n^h(\gamma)/\delta) + 5\underbrace{\sum_{i=1}^n \frac{|\hat{f}^h(x_i^h,a_i^h)-(\cT^h\Bar{f}^{h+1})(x_i^h,a_i^h)|\cdot|\zeta_i^h|}{(\sigma_i^h)^2}}_{(a)}\nonumber\\
&\qquad + 10(\gamma + 2\beps)\cdot\left((\gamma + 2\beps)n + \sqrt{nC_1(n,\zeta)}\right),
\#
where $C_1(n,\zeta)=2(\zeta^2 + 2n + 3\ln(2/\delta))$. 

According to Lemma \ref{lm:converge_weight}, the term (a) can be controlled by the weight design:
\$
&\sum_{i=1}^n|\zeta_i^h|\cdot\frac{|\hat{f}^h(x_i^h,a_i^h)-(\cT^h\Bar{f}^{h+1})(x_i^h,a_i^h)|}{(\sigma_i^h)^2}\\ 
&\qquad\le \zeta^h \sup_i \frac{|\hat{f}^h(x_i^h,a_i^h)-(\cT^h\Bar{f}^{h+1})(x_i^h,a_i^h)|}{ (\sigma_i^h)^2}\\
&\qquad\le 2\alpha\zeta^h\sqrt{\lambda + \sum_{i=1}^n \frac{\big(\hat{f}^h(x_i^h,a_i^h)-(\cT^h\Bar{f}^{h+1})(x_i^h,a_i^h)\big)^2}{(\sigma_i^h)^2}},
\$
where the second inequality is obtained since $\sigma_i^h$ satisfies \eqref{eq:approximate_weight}.
Taking this result back into \eqref{eq:sum_hat_fkh_Tb_barf0}, we finally get for all $h\in[H]$,
\$
&\sum_{i=1}^n \frac{\big(\hat{f}^h(x_i^h,a_i^h)-(\cT^h\Bar{f}^{h+1})(x_i^h,a_i^h)\big)^2}{(\sigma_i^h)^2}\\
&\qquad\le 10\ln(2HN_n^h(\gamma)/\delta) + 10\alpha\zeta^h\beta^h + 5\gamma\zeta + 10(\gamma + 2\beps)\cdot((\gamma + 2\beps)n + \sqrt{nC_1(t,\zeta)})\\
&\qquad= 10\ln(2HN_n^h(\gamma)/\delta) + 10\alpha\zeta^h\beta^h + 5\gamma\zeta + 10(2\beta^{h+1}+3)^2\gamma^2n + 10(2\beta^{h+1}+3)\gamma\sqrt{nC_1(n,\zeta)},
\$
where the last equality uses $\beps=(1+\beta^{h+1})\gamma$.
Therefore, it follows that with probability at least $1-\delta$,
\$
&\left(\sum_{i=1}^{n}\frac{\left(\hat{f}_n^h(x_i^h,a_i^h)-(\cT^h f_n^{h+1})(x_i^h,a_i^h)\right)^2}{(\sigma_i^h)^2} + \lambda\right)^{1/2}\\
&\qquad\le \left(\sum_{i=1}^{n}\frac{\left(\hat{f}_n^h(x_i^h,a_i^h)-(\cT^h\Bar{f}_n^{h+1})(x_i^h,a_i^h)\right)^2}{(\sigma_i^h)^2}\right)^{1/2} + \sqrt{n}\beps + \sqrt{\lambda}\\
&\qquad\le \Big(10\ln(2HN_n^h(\gamma)/\delta) + 10\alpha\zeta^h\beta^h + 5\gamma\zeta + 10(2\beta^{h+1}+3)^2\gamma^2n\\
&\qquad\qquad + 10(2\beta^{h+1}+3)\gamma\sqrt{nC_1(n,\zeta)}\Big)^{1/2} + (\beta^{h+1}+1)\gamma\sqrt{n} + \sqrt{\lambda}\\
&\qquad\le \left(12\lambda + 12\ln(2HN_n^h(\gamma)/\delta) + 24\alpha\zeta^h\beta^h + 12(5\beta^{h+1}\gamma)^2n + 60\beta^{h+1}\gamma\sqrt{nC_1(n,\zeta)}\right)^{1/2} \le \beta^h.
\$
Therefore, we complete the proof.
\end{proof}

\subsection{Step III: Bound the Suboptimality}
\begin{proof}[Proof of Theorem \ref{th:mdp}]
We know from Lemma \ref{lm:Regret_Decomposition_mdp} and Lemma \ref{lm:Confidence_Radius_mdp} that with probability at least $1-\delta$,
\#\label{eq:sum_beta_E_pi*_b(x,a)}
\subopt(\hat{\pi}, x) &\le 2\sum_{h=1}^H\beta^h\E_{\pi_*}\big[ b^h(x^h,a^h)\,\big|\, x^1=x\big]\notag\\
& = 2\sum_{h=1}^H \beta^h \E_{\pi_*} \bigg[ \frac{b^h(x^h,a^h)}{\sigma^h(x^h,a^h)} \cdot \mathbbm{1}(\sigma^h(x^h,a^h)=1)\notag\\
&\qquad + \frac{b^h(x^h,a^h)}{\sigma^h(x^h,a^h)} \cdot \sigma^h(x^h, a^h) \cdot \mathbbm{1}(\sigma^h(x^h, a^h)>1) \,\bigg|\, x^1=x \bigg]\notag\\
& \le 2\sum_{h=1}^H \beta^h \E_{\pi_*} \bigg[ \frac{b^h(x^h,a^h)}{\sigma^h(x^h,a^h)} + \Big(\frac{b^h(x^h,a^h)}{\sigma^h(x^h,a^h)}\Big)^2 \cdot \frac{1}{\alpha } \,\bigg|\, x^1=x \bigg],
\#
where the last inequality is deduced since from the definition of $\sigma^h(x^h,a^h)$ in \eqref{eq:sigma^h(x^h,a^h)^2}, we get for $\sigma^h(x^h,a^h)>1$ that
\$
\sigma^h(x^h,a^h)&=\frac{1}{\sigma^h(x^h,a^h)}\cdot\sup_{f,f'\in\hcF^h} \frac{|f(x^h,a^h) - f'(x^h,a^h)|/\alpha}{\sqrt{\lambda + \sum_{i=1}^n(f(x_i^h,a_i^h) - f'(x_i^h,a_i^h))^2/(\sigma_i^h)^2}}= \frac{b^h(x^h,a^h)}{\alpha\sigma^h(x^h,a^h)}.
\$
From the definition of the weighted coverage condition \eqref{eq:coverage coefficient weighted}, we have
\$
\E_{\pi_*} \bigg[\Big(\frac{b^h(x^h,a^h)}{\sigma^h(x^h,a^h)}\Big)^2 \,\bigg|\, x^1=x \bigg] = \frac{\CC^{\sigma}(\lambda,\hcF,\cZ_n^H)}{n}.
\$
Combining the above equation  and $\E X\le \sqrt{\E X^2}$, we can bound \eqref{eq:sum_beta_E_pi*_b(x,a)} by
\$
&\sum_{h=1}^H \beta^h \E_{\pi_*} \bigg[ \frac{b^h(x^h,a^h)}{\sigma^h(x^h,a^h)} + \Big(\frac{b^h(x^h,a^h)}{\sigma^h(x^h,a^h)}\Big)^2 \cdot \frac{1}{\alpha } \,\bigg|\, x^1=x \bigg]\\
&\qquad\le\sum_{h=1}^H \beta^h\bigg[ \sqrt{\frac{\CC^{\sigma}(\lambda,\hcF,\cZ_n^H)}{n}} + \frac{\CC^{\sigma}(\lambda,\hcF,\cZ_n^H)}{n}\cdot\frac{1}{\alpha }\bigg]\\
&\qquad \le \sqrt{\frac{\CC^{\sigma}(\lambda,\hcF,\cZ_n^H)}{n}} \sum_{h=1}^H \beta^h + \frac{\CC^{\sigma}(\lambda,\hcF,\cZ_n^H)}{n }\cdot\frac{\sum_{h=1}^H\beta^h}{\alpha}\\
&\qquad \le \sqrt{\frac{\CC^{\sigma}(\lambda,\hcF,\cZ_n^H)}{n}} \cdot c_{\beta}\big(\alpha\zeta + H\sqrt{\ln N}\big) + \frac{\CC^{\sigma}(\lambda,\hcF,\cZ_n^H)}{n }\cdot c_{\beta}\Big(\zeta + \frac{H\sqrt{\ln N}}{\alpha}\Big).
\$
By choosing $\alpha=H\sqrt{\ln N}/\zeta$, we can obtain the result:
\$
\subopt(\hat{\pi}, x) = \tcO\bigg(\frac{3H\sqrt{\CC^{\sigma}(\lambda,\hcF,\cZ_n^H)\cdot\ln N}}{\sqrt{n}} + \frac{\zeta\cdot\CC^{\sigma}(\lambda,\hcF,\cZ_n^H)}{n }\bigg).
\$
Ultimately, we can invoke Lemma \ref{lm:Connections between Weighted and Unweighted Coefficient} to obtain that with probability at least $1-2\delta$,
\$
\subopt(\hat\pi, x) = \tcO\bigg(3H\sqrt{\frac{\ln N}{C(\cF,\mu)n}} + \frac{\zeta}{n C(\cF,\mu)}\bigg),
\$  
which concludes the proof.
\end{proof}

\section{Proofs of Auxiliary Results}

\subsection{Proof of Lemma \ref{lm:converge_weight}}\label{ss:Proof of Lemma converge_weight}
To begin with, we demonstrate that the uncertainty weight iteration (Algorithm \ref{alg:wi}) converges, thus satisfying the approximate condition in \eqref{eq:approximate_weight}.

\begin{proof}[Proof of Lemma \ref{lm:converge_weight}]
We will demonstrate this result via the convergence of monotone real number sequences. To begin with, we prove the monotonicity of $\{(\sigma_i^t)^2\}_{t=0}^{\infty}$ by induction. When $\tau=0$, we know that $(\sigma_i^1)^2\ge 1 = \sigma_i^0$ for $i\in[n]$. When $\tau=t$, assume that $(\sigma_i^t)^2\ge(\sigma_i^{t-1})^2$ for all $i\in[n]$, which implies that
\begin{align*}
&\sup_{f,f'\in\cF}\frac{|f(x_i,a_i) - f'(x_i,a_i)|/(\alpha )}{\sqrt{\lambda + \sum_{j=1}^n(f(x_j,a_j) - f'(x_j,a_j))^2/(\sigma_j^{t})^2}}\\
&\ge \sup_{f,f'\in\cF}\frac{|f(x_i,a_i) - f'(x_i,a_i)|/(\alpha )}{\sqrt{\lambda + \sum_{j=1}^n(f(x_j,a_j) - f'(x_j,a_j))^2/(\sigma_j^{t-1})^2}}.
\end{align*}
Therefore, we get $(\sigma_i^{t+1})^2 \ge (\sigma_i^t)^2$ for all $i\in[n]$. Then, when $t=T$, we deduce from $(\sigma_i^{T+1})^2 \ge (\sigma_i^T)^2$ that for each $i\in[n]$,
\$
(\sigma_i^T)^2 \le (\sigma_i^{T+1})^2 \le \max\Big(1,\sup_{f,f'\in\cF}\frac{|f(x_i,a_i) - f'(x_i,a_i)|/\alpha}{\sqrt{\lambda + \sum_{j=1}^n(f(x_j,a_j) - f'(x_j,a_j))^2/(\sigma_j^N)^2}}\Big),
\$
which implies the second inequality of \eqref{eq:approximate_weight}.

Then, we obtain the upper bounds of each sequence: for $i\in[n]$ and any $t\ge0$
$$
\sup_{f,f'\in\cF}\frac{|f(x_i,a_i) - f'(x_i,a_i)|/\alpha}{\sqrt{\lambda + \sum_{j=1}^n(f(x_j,a_j) - f'(x_j,a_j))^2/(\sigma_j^{t})^2}} \le \frac{1}{\alpha \sqrt{\lambda}},
$$
where we use $f(\cdot)\in[0,1]$ for any $f\in\cF$. 
Thus, each $\{(\sigma_i^t)^2\}_{t=0}^{\infty}$ has a $\max(1, 1/(\alpha \sqrt{\lambda}))$ upper bound. 

According to the convergence of monotone real number sequences, the sequence $\{(\sigma_i^t)^2\}_{t=1}^{\infty}$ converges for all $i\in[n]$, which implies that $\{\log((\sigma_i^t)^2)\}_{t=1}^{\infty}$ converges. We know from the definition of convergence that for any $\mu>0$, there exists an $N(\mu)$ such that for any $t\ge N$,
\$
\log\big((\sigma_i^{t+1})^2\big) - \log\big((\sigma_i^t)^2\big) \le \log(1+ \mu),
\$
which implies that
\#\label{eq:frac_sigma}
\frac{(\sigma_i^{t+1})^2}{(\sigma_i^n)^2}\le 1+\mu.
\#

For any $t\ge N$, if $\sigma_i^{t+1}=1$, we have from the monotonicity that $\sigma_i^t=1$, thus satisfying the first inequality of \eqref{eq:approximate_weight}. If $\sigma_i^{n+1}>1$, we deduce from \eqref{eq:frac_sigma} that
$$
(\sigma_i^t)^2 \ge \frac{1}{1+\mu} (\sigma_i^{t+1})^2 = \frac{1}{1+\mu} \sup_{f,f'\in\cF}\frac{|f(x_i,a_i) - f'(x_i,a_i)|/(\alpha )}{\sqrt{\lambda + \sum_{j=1}^n(f(x_j,a_j) - f'(x_j,a_j))^2/(\sigma_j^t)^2}}.
$$
Hence, the inequality is proved by taking $\mu=1$ and stop the iteration at round $t=N+1$.
\end{proof}

\subsection{Connection between Coverage Coefficients}
In this part, we state that in the linear MDP, the coverage condition in \citet{jin2021pessimism} implies our coverage assumption. Recall the coverage coefficient defined in \eqref{eq:coverage coefficient unweighted}:
\$
\CC(\lambda,\hcF,\cZ_n^H) = \max_{h\in[H]}\E_{\pi_*}\bigg[\sup_{f,f'\in\hcF^h}\frac{n(f(x^h,a^h) - f'(x^h,a^h))^2}{\lambda + \sum_{i=1}^n(f(x_i^h,a_i^h) - f'(x_i^h,a_i^h))^2} \,\bigg|\, x^1=x\bigg].
\$
When the function space $\cF^h$ is embedded into a $d$-dimensional vector space: $\cF^h=\{\langle w(f), \phi(\cdot) \rangle : z\rightarrow\rR\}$, where $z$ denotes the state-action pair $(x,a)$. Then, we define $\Lambda^h=\lambda I+\sum_{i=1}^n\phi(z_i^h)\phi(z_i^h)^\top$.

The coverage condition in \citet{jin2021pessimism} assumes that there exists a constant $c^{\dagger}$ such that for all $h\in[H]$,
\#\label{eq:coverage_linear_jin}
\Lambda^h\succeq I+c^{\dagger}n\E_{\pi_*}\big[\phi(z^h) \phi(z^h)^\top\,\big|\,x^1=x\big].
\#
\begin{lemma}\label{lm:Connection between Coverage Coefficients_linear}
In the linear setting with dimension $d$, if the coverage condition in \eqref{eq:coverage_linear_jin} holds with a finite constant $c^{\dagger}>0$, the coverage coefficient in \eqref{eq:coverage coefficient unweighted} is also finite: 
\$
\CC(\lambda,\hcF,\cZ_n^H) \le \frac{d}{c^{\dagger}} <\infty.
\$
\end{lemma}
\begin{proof}
We deduce that
\$
&\sup_{f,f'\in\hcF^h}\frac{n\big(\langle w(f)-w(f'), \phi(z) \rangle\big)^2}{\lambda + \sum_{i=1}^n\big(\langle w(f)-w(f'), \phi(z_i^h) \rangle\big)^2}\\
&\qquad\le \sup_{f,f'\in\hcF^h}\frac{n\big(\langle w(f)-w(f'), \phi(z) \rangle\big)^2}{\big(w(f)-w(f')\big)^\top \Lambda^h \big(w(f)-w(f')\big)}\\
&\qquad\le n \phi(z)^\top (\Lambda^h)^{-1} \phi(z),
\$
where the second inequality uses
\$
\langle w(f)-w(f'), \phi(z) \rangle \le \sqrt{\big(w(f)-w(f')\big)^\top \Lambda^h \big(w(f)-w(f')\big)} \cdot \sqrt{\phi(z)^\top (\Lambda^h)^{-1} \phi(z)}.
\$
Thus, the coverage coefficient \eqref{eq:coverage coefficient unweighted} for the linear model is bounded by
\$
\CC(\lambda,\hcF,\cZ_n^H) &\le \max_{h\in[H]}\E_{\pi_*}\bigg[\sup_{f,f'\in\hcF^h}\frac{n\big(\langle w(f)-w(f'), \phi(z) \rangle\big)^2}{\lambda + \sum_{i=1}^n\big(\langle w(f)-w(f'), \phi(z_i^h) \rangle\big)^2} \,\bigg|\, x^1=x\bigg]\\
&\le \max_{h\in[H]}\E_{\pi_*}\big[n\phi(z^h)^\top (\Lambda^h)^{-1} \phi(z^h)\,\big|\,x^1=x\big].
\$
If the sufficient ``coverage'' in \eqref{eq:coverage_linear_jin} holds, then, we have
\$
\CC(\lambda,\hcF,\cZ_n^H) &\le \frac{1}{c^{\dagger}}\max_{h\in[H]}
\text{Tr}\Big\{(\Lambda^h)^{-1}\cdot\big(I+c^{\dagger}n\E_{\pi_*}\big[\phi(z^h) \phi(z^h)^\top\,\big|\,x^1=x\big]\big)\Big\}\\
&\le \frac{d}{c^{\dagger}} < \infty,
\$
which concludes the proof.
\end{proof}

\subsection{Proof of Lemma \ref{lm:Connections between Weighted and Unweighted Coefficient}}\label{s:Connections between Weighted and Unweighted Coefficient}
We use the shorthand notation $z=(x,a)$ for any $(x,a)\in\cX\times\cA$.
Recall the definition of weighted coverage coefficient in 
\eqref{eq:coverage coefficient weighted}:
\$
\CC^{\sigma}(\lambda,\hcF,\cZ_n^H) = \max_{h\in[H]}\E_{\pi_*}\bigg[\sup_{f,f'\in\hcF^h}\frac{n(f(z^h) - f'(z^h))^2/(\sigma^h(z^h))^2}{\lambda + \sum_{i=1}^n(f(z_i^h) - f'(z_i^h))^2/(\sigma_i^h)^2} \,\bigg|\, x^1=x\bigg],
\$
where 
\$
(\sigma^h(z^h))^2 = \max\bigg(1, \sup_{f,f'\in\hcF^h} \frac{|f(z^h) - f'(z^h)|/\alpha}{\sqrt{\lambda + \sum_{i=1}^n(f(z_i^h) - f'(z_i^h))^2/(\sigma_i^h)^2}}\bigg).
\$
In the sequel, we will control the weighted coverage coefficient $\CC^{\sigma}(\lambda,\hcF,\cZ_n^H)$ by $C(\cF,\mu)$ in Assumption \eqref{as:Well-Explored Dataset}.

Now, we present the proof of Lemma \ref{lm:Connections between Weighted and Unweighted Coefficient}.
\begin{proof}[Proof of Lemma \ref{lm:Connections between Weighted and Unweighted Coefficient}]
For convenience, we use the short-hand notation
\$
\psi(z^h) = \sup_{f,f'\in\hcF^h} \frac{(f(z^h) - f'(z^h))^2 / (\sigma^h(z^h))^2}{\lambda + \sum_{i=1}^n(f(z_i^h) - f'(z_i^h))^2/(\sigma_i^h)^2}.
\$
Let $f_{z^h},f'_{z^h}$ be the functions that maximize
\$
\frac{(f(z^h) - f'(z^h))^2 / (\sigma^h(z^h))^2}{\lambda + \sum_{i=1}^n(f(z_i^h) - f'(z_i^h))^2/(\sigma_i^h)^2}.
\$
Then, we can rewrite
\$
\psi(z^h) = \frac{(f_{z^h}(z^h) - f'_{z^h}(z^h))^2 / (\sigma^h(z^h))^2}{\lambda + \sum_{i=1}^n(f_{z^h}(z_i^h) - f'_{z^h}(z_i^h))^2/(\sigma_i^h)^2}.
\$
Since
\$
(\sigma^h(z^h))^2 \ge  \sup_{f,f'\in\hcF^h} \frac{|f(z^h) - f'(z^h)|/\alpha}{\sqrt{\lambda + \sum_{i=1}^n(f(z_i^h) - f'(z_i^h))^2/(\sigma_i^h)^2}},
\$
we find that
\#\label{eqap:sigma(zh)2_bound}
\sigma^h(z^h) \ge & \frac{1}{\alpha} \sup_{f,f'\in\hcF^h} \frac{|f(z^h) - f'(z^h)|/\sigma^h(z^h)}{\sqrt{\lambda + \sum_{i=1}^n(f(z_i^h) - f'(z_i^h))^2/(\sigma_i^h)^2}}\notag\\
= & \frac{1}{\alpha}\cdot \sqrt{\psi(z^h)}.
\#

Then, we will derive the upper bound for $\{\sigma_i^h\}_{i=1}^n$. For each $i\in[n]$, we have from Lemma \ref{lm:converge_weight} that
\$
(\sigma_i^h)^2 \le \max\bigg(1, \sup_{f,f'\in\cF^h} \frac{|f(z_i^h) - f'(z_i^h)|/\alpha}{\sqrt{\lambda + \sum_{j=1}^n(f(z_j^h) - f'(z_j^h))^2/(\sigma_j^h)^2}}\bigg).
\$
Let $f_i,f_i'$ be the maximizer of 
\$
\frac{|f(z_i^h) - f'(z_i^h)|/\alpha}{\sqrt{\lambda + \sum_{j=1}^n(f(z_j^h) - f'(z_j^h))^2/(\sigma_j^h)^2}}\bigg).
\$
Hence, by using Assumption \ref{as:Well-Explored Dataset}, we get
\$
\frac{|f_i(z_i^h) - f_i'(z_i^h)|/\alpha}{\sqrt{\lambda + \sum_{j=1}^n(f_i(z_j^h) - f_i'(z_j^h))^2/(\sigma_j^h)^2}} \le & \frac{\|f_i - f_i'\|_{\infty}/\alpha}{\sqrt{\lambda + nC(\cF,\mu)\|f_i - f_i'\|_{\infty}^2/\max_j(\sigma_j^h)^2}}\\
\le & \frac{\max_j\sigma_j^h}{\alpha\sqrt{nC(\cF,\mu)}},
\$
which implies that for all $i\in[n]$,
\$
(\sigma_i^h)^2 \le & \max\Big(1, \frac{\max_j\sigma_j^h}{\alpha\sqrt{nC(\cF,\mu)}}\Big)\\
\le & \max_j\sigma_j^h \max\Big(1, \frac{1}{\alpha\sqrt{nC(\cF,\mu)}}\Big).
\$
By taking the maximum over $i\in[n]$, we have
\#\label{eqap:sigmaih_bound}
\max_i\sigma_i^h \le \max\Big(1, \frac{1}{\alpha\sqrt{nC(\cF,\mu)}}\Big).
\#

There are two situations.

First, if $\alpha\sqrt{nC(\cF,\mu)}\ge 1$, i.e., $\zeta \le H\sqrt{nC(\cF,\mu)\ln N_n(\gamma)}$, we know from \eqref{eqap:sigmaih_bound} that $\sigma_i^h=1$ for all $i\in[n]$. It follows that
\$
\psi(z^h) \le & \frac{(f_{z^h}(z^h) - f'_{z^h}(z^h))^2}{\lambda + \sum_{i=1}^n(f_{z^h}(z_i^h) - f'_{z^h}(z_i^h))^2}\\
\le & \frac{(f_{z^h}(z^h) - f'_{z^h}(z^h))^2}{nC(\cF,\mu)\|f - f'\|_{\infty}^2}\\
\le & \frac{1}{nC(\cF,\mu)}.
\$

Second, if $\alpha\sqrt{nC(\cF,\mu)} < 1$, i.e., $\zeta > H\sqrt{nC(\cF,\mu)\ln N_n(\gamma)}$,
Thus, by taking the upper bound for weights into $\psi(z^h)$, we obtain
\$
\psi(z^h) \le & \frac{(f_{z^h}(z^h) - f'_{z^h}(z^h))^2 / (\sigma^h(z^h))^2}{\lambda + \alpha^2nC(\cF,\mu)\sum_{i=1}^n(f_{z^h}(z_i^h) - f'_{z^h}(z_i^h))^2}\\
\le & \frac{(f_{z^h}(z^h) - f'_{z^h}(z^h))^2 \cdot \alpha^2/ \psi(z^h)}{\lambda + \alpha^2nC(\cF,\mu)\sum_{i=1}^n(f_{z^h}(z_i^h) - f'_{z^h}(z_i^h))^2}\\
\le & \frac{(f_{z^h}(z^h) - f'_{z^h}(z^h))^2 \cdot \alpha^2/ \psi(z^h)}{\lambda + \alpha^2n^2(C(\cF,\mu))^2\|f_{z^h} - f'_{z^h}\|_{\infty}^2}\\
\le & \frac{1}{\psi(z^h)n^2(C(\cF,\mu))^2},
\$
which implies that
\$
\psi(z^h) \le \frac{1}{nC(\cF,\mu)}.
\$

In conclusion, we have
\$
\CC^{\sigma}(\lambda,\hcF,\cZ_n^H) = n\max_{h} \E_{\pi_*} [\psi(z^h)] \le \frac{1}{C(\cF,\mu)},
\$
which concludes the proof.
\end{proof}

\paragraph{Interpretation of Assumption \ref{as:Well-Explored Dataset} in linear MDPs.}

Now, we illustrate the condition \eqref{eq:condition of lm:Connections between Weighted and Unweighted Coefficient} with the linear model. When the function space $\cF^h$ can be embedded into a $d$-dimensional vector space: $\cF^h=\{\langle w(f), \phi(\cdot) \rangle : z\rightarrow\rR\}$, the condition \eqref{eq:condition of lm:Connections between Weighted and Unweighted Coefficient} becomes: for any two distinct $f,f'\in\bcF^h$,
\#\label{eq:linear condition of lm:Connections}
\sum_{z^h} \mu^h(z) \big(f(z^h)-f'(z^h)\big)^2 &= \big(w(f)-w(f')\big)^{\top} \bar\Lambda^h \big(w(f)-w(f')\big) \ge C(\hcF,\mu)\big\|w(f)-w(f')\big\|^2,
\#
where we define $\bar\Lambda^h=\E_{z^h\sim\mu^h} [\phi(z^h)\phi(z^h)^\top]$ and $\mu^h$ is the data empirical distribution.

In the following lemma, we demonstrate that the above condition holds as long as the learner has excess to a well-explored dataset \eqref{eq:minimum eigenvalue condition}, which is a wildly-adopted assumption in the literature of offline linear MDPs \citep{duan2020minimax,wang2020statistical,zhong2022pessimistic}. Note that $d^{-1}$ is the largest possible order of the minimum eigenvalue since for any data distribution $d^h$, $\sigma_{\min}(\E_{z^h\sim d^h} [\phi(z)\phi(z)^\top]) \le d^{-1}$ by using $\|\phi(z)\|\le 1$ for any $z\in\cX\times\cA$.

\begin{lemma}\label{lm:minimum eigenvalue condition and condition for W and UW connections}
In the linear setting, if we assume that the data distributions $ d^h$ satisfy
the following minimum eigenvalue condition: there exists an absolute constant $\bar{c}>0$ such that
\#\label{eq:minimum eigenvalue condition}
\sigma_{\min}\Big(\E_{z^h\sim d^h} [\phi(z)\phi(z)^\top]\Big)=\frac{\bar{c}}{d},
\#
the dataset $\cD$ consists of $n\ge 128d^2\bar{c}^{-2}\log(d/(2\delta)$ independent trajectories,
then, the condition \eqref{eq:linear condition of lm:Connections} with $C(\hcF,\mu)=\bar{c}'/(2d)$ will holds with probability at least $1-\delta$, where $\bar{c}'$
\end{lemma}
\begin{proof}
To begin with, we aim to prove that the empirical matrix $\bar\Lambda^h=n^{-1}\sum_{i=1}^n \phi(z_i^h)\phi(z_i^h)^\top$ is positive definite with high probability. Since $\|\phi(z)\|\le 1$ for any $z\in\cX\times\cA$, we have for each $i\in[n]$,
\$
\big(\phi(z_i^h)\phi(z_i^h)^\top - \E_{z^h\sim d^h} [\phi(z)\phi(z)^\top] \big)^2 \preceq (2I)^2.
\$
By invoking the matrix Hoeffding's concentration in Lemma \ref{lm:Matrix Hoeffding Concentration} with $X_i=\phi(z_i^h)\phi(z_i^h)^\top - \E_{z^h\sim d^h} [\phi(z)\phi(z)^\top]$, $A_i=2I$, $\sigma^2=4n$, we obtain
\$
\Pb\bigg(\Big\|\sum_{i=1}^n\big(\phi(z_i^h)\phi(z_i^h)^\top - \E_{z^h\sim d^h} [\phi(z)\phi(z)^\top]\big)\Big\|_\op \Big)\ge t\bigg) \le 2d\cdot e^{-t^2/(32n)}.
\$
For any $\delta>0$, by taking $t=\sqrt{32n\log(d/\delta)}$, we have with probability at least $1-\delta$, 
\$
\Big\|\frac{1}{n}\sum_{i=1}^n\big(\phi(z_i^h)\phi(z_i^h)^\top - \E_{z^h\sim d^h} [\phi(z)\phi(z)^\top]\big)\Big\|_\op \le \sqrt{\frac{32\log(d/(2\delta))}{n}}.
\$
Hence, whenever $n\ge 128d^2\bar{c}^{-2}\log(d/(2\delta)$, by combing the above result  with \eqref{eq:minimum eigenvalue condition}, we have with probability at least $1-\delta$, 
\#\label{eq:empirical minimum eigenvalue condition}
\sigma_{\min}\big(\bar\Lambda^h\big)=\sigma_{\min}\Big(\frac{1}{n}\sum_{i=1}^n\phi(z_i^h)\phi(z_i^h)^\top\Big)\ge\frac{\bar{c}}{2d}.
\#
Then, since the cardinality of the cover $\bcF^h$ is $\tcO(d)$, we can define $a=\min_{f,f'\bcF^h}\|w(f)-w(f')\|^2$. Thus, by using \eqref{eq:empirical minimum eigenvalue condition}, the condition \eqref{eq:linear condition of lm:Connections} is inferred: for any $f,f'\in\bcF^h$ with probability at least $1-\delta$,
\$
\big(w(f)-w(f')\big)^{\top} \bar\Lambda^h \big(w(f)-w(f')\big) \ge \frac{\bar{c}}{2d} \big\|w(f)-w(f')\big\|^2 \ge \frac{a\bar{c}}{2d} = C(\hcF,\mu)\big\|w(f)-w(f')\big\|^2.
\$
\end{proof}

\subsection{Lower Bound for Linear MDPs with Corruption}\label{ss:Lower Bound for Linear MDPs with Corruption}
\begin{proof}[Proof of Theorem \ref{th:Lower Bound}]
For any dimension $d$, step horizon $H$ and corruption level $\zeta$, we construct a tabular MDP with action number $A>2$ and state number $S$ such that $SA=d$. The MDP is represented by a tree with depth $L$ and $S=(A^L-1)/A$ nodes. The first level has $1$ nodes, the second level has $A$ nodes, $\ldots$, the last level has $A^{L-1}$ nodes. Each state corresponds to a node, and each action corresponds to an edge. The agent starts from the first level. For each state in the first $L-1$ levels, each action leads to $A$ child nodes uniformly. All the leaf nodes are absorbing states. 

During the data selection process, we proceed the behavior policy $\pi^v$ such that the actions are chosen uniformly. If the number of trajectories is $n$, for each level $l\in[L]$, the expected times that each state-action pair $(x,a)$ is visited is at least
\$
\frac{n}{A^{l-1}}.
\$
By using Azuma–Hoeffding inequality (Lemma \ref{lemma:azuma}) and $n\ge 2A^{L-1}\log 2$, with probability at least $1/2$, the times that each state-action pair $(x,a)$ is visited is at least
\$
\frac{n}{2A^{l-1}},
\$
and at most
\$
\frac{3n}{2A^{l-1}}.
\$

We consider two MDPs $M$ and $M'$ with the same transition structure and different reward function. For MDP $M$, let $r(s^*,a^*)=H^{-1}\text{Bernoulli}(A^{L-1}\epsilon/2)$ on one particular state $(s^*,a^*)$, where $s^*$ is a leaf state and $a^*$ is a self-loop action. The rewards of every other $(s,a)$ are zero. Let $(s',a')$ be the state-action pair that is visited least often at the last level. By the pigeonhole principle, we have
\$
N(s',a') \le \frac{n(H-L+1)}{A^{L-1}}.
\$
Let $r'(s^*,a^*)=H^{-1}\text{Bernoulli}(A^{L-1}\epsilon/2)$ and $r'(s',a')=H^{-1}\text{Bernoulli}(A^{L-1}\epsilon)$ and $0$ for every other pair. Therefore no policy can be better than $(H-L+1)A^{L-1}\epsilon/4H$-optimal on both $M$ and $M'$.

Suppose that the adversary corrupts all the non-zero rewards $r(s,a)$ to zero if $(s,a)\ne(s^*,a^*)$. Conditioning on $N(s',a') \le \frac{n(H-L+1)}{A^{L-1}}$, with probability at least $1/2$, the amount of corruption is at most
\$
\frac{1}{H}\cdot A^{L-1}\epsilon N(s',a') \le \epsilon n \le \zeta,
\$
which means that the adversary can perturb every positive rewards on $(s',a')$ to $0$ as long as $C= n\epsilon$. Thus, we can write the suboptimality as
\$
\Omega\big(\frac{A^{L-1}\zeta}{n}\big).
\$

Moreover, we can lower bound $C(\cF,\mu)$ in this case and only consider $H\ge L$: for any $f,f'\in\cF^h$,
\$
\|f - f'\|_{\infty}^2 \le & \sum_{z^h} (f(z^h) - f'(z^h))^2\\
\le & \frac{2A^{L-1}}{n}\sum_{z^h}(f(z^h) - f'(z^h))^2\\
= & 2A^{L-1} \E_{\mu^h}\big[(f(z^h) - f'(z^h))^2\big],
\$
which implies that $C(\cF,\mu)\ge \frac{1}{2A^{L-1}}$. Thus, the learner must suffer from at least $\Omega(A^{L-1}\zeta/n)=\Omega(\zeta/nC(\cF,\mu))$ suboptimality with probability at least $1/4$.
\end{proof}

\subsection{Relationship between the Bootstrapped Uncertainty and Bonus}\label{ss:Relationship between the Bootstrapped Uncertainty and Bonus}
In the sequel, we discuss the relationship between the bootstrapped uncertainty and the bonus function by considering linear function approximation. Using $z$ to denote the feature variable of state-action pair $(x,a)$, we estimate the Q-value function $Q_\phi$ by $Q^h_{w^h}(z)=(w^h)^\top \phi(z)$ to minimize the Bellman error target with weights:
\#\label{eq:WLS}
\hw^h = \argmin_{w\in\rR^d} \sum_{z^h_i\in\cD} \frac{\big((\widehat{\mathcal{T}}^h Q_w)(z^h_i) - w^\top \phi(z_i^h)\big)^2}{\sigma(z^h_i)} ,
\#
where $\{\sigma(z^h_i)>0\}$ is a group of predetermined weights and for any $z^h_i\in\cD$,
\$
(\widehat{\mathcal{T}}^h Q_w)(z^h_i) = r^h(z^h_i) + V_w^{h+1}(z^{h+1}_i).
\$
Additionally, we define the noise in this weighted least square problem as $\epsilon=\widehat{\mathcal{T}}^h Q_w(x,a) - Q^h_w(z)$.

\paragraph{Bonus functions.} 

In the traditional linear MDP, we often use the following term as the bonus function:
\#\label{eq:bonus_linear}
b_L^h(z) = \sqrt{\phi(z)^\top (\Lambda^h)^{-1} \phi(z)},
\#
where $\Lambda^h=\sum_{i=1}^n \phi(z_i^h)\phi(z_i^h)^\top/(\sigma_i^h)^2$, and we use the shorthand notation for any matrix $A$ and vector $x$: $\|x\|_A=\sqrt{x^\top A x}$.

The bonus function in the general form \eqref{eq:bonus} turns into the following form under the linear setting:
\$
b^h(z) = \sup_{f,f'\in\hcF^h}\frac{\big|\big(w(f)-w(f')\big)^\top \phi(z)\big|}{\sqrt{\lambda + \sum_{i=1}^n \big((w(f)-w(f'))^\top \phi(z_i^h)/\sigma_i^h\big)^2}}.
\$
We demonstrate that the linear and general forms of bonus functions are almost equivalent under mild conditions.
\begin{lemma}\label{lm:equivalence bonus linear and general}
Under the linear MDP, if the function space is broad enough such that for any $z\in\cX\times\cA$, there exists $f,f'\in\hcF^h$ satisfying that $w(f)-w(f')$ and $(\Lambda^h)^{-1}\phi(z)$ are in the same direction and not too close, i.e., for some $\alpha>0$,
\$
w(f)-w(f') = \alpha\cdot (\Lambda^h)^{-1}\phi(z),\quad \text{and}\quad \|w(f)-w(f')\|_{\Lambda^h}^2 \ge \lambda,
\$
then, we have for any $z\in\cX\times\cA$,
\$
\frac{b_L^h(z)}{\lambda^{1/4}+1} \le b^h(z)\le b_L^h(z).
\$
\end{lemma}
\begin{proof}
First, we will prove $b^h(z)\le b_L^h(z)$. By the definition of $b^h$, we have
\$
b^h(z) &= \sup_{f,f'\in\hcF^h}\frac{\big|\big(w(f)-w(f')\big)^\top \phi(z)\big|}{\sqrt{\lambda + \sum_{i=1}^n \big((w(f)-w(f'))^\top (\phi(z_i^h)/\sigma_i^h\big)^2}}\notag\\
&\le \sup_{f,f'\in\hcF^h}\frac{\big|\big(w(f)-w(f')\big)^\top \phi(z)\big|}{\sqrt{\big(w(f)-w(f')\big)^\top \Lambda^h \big(w(f)-w(f')\big)}}\\
&\le \sqrt{\phi(z)^\top (\Lambda^h)^{-1} \phi(z)} = b_L^h(z)\notag.
\$
Then, we will prove $b^h(z)\ge b_L^h(z)$. By the assumption, for any $z\in\cX\times\cA$, there exists $f_1,f_2\in\hcF^h$ such that for some $\alpha>0$,
\$
w(f_1)-w(f_2) = \alpha \cdot (\Lambda^h)^{-1}\phi(z),
\$
which implies that
\$
(\Lambda^h)^{1/2}(w(f_1)-w(f_2)) = \alpha \cdot (\Lambda^h)^{-1/2}\phi(z).
\$
Then, we have
\#\label{eq:b^hz_lower bound}
b^h(z) &= \sup_{f,f'\in\hcF^h}\frac{\big|\big(w(f)-w(f')\big)^\top \phi(z)\big|}{\sqrt{\lambda+\big(w(f)-w(f')\big)^\top \Lambda^h \big(w(f)-w(f')\big)}}\notag\\
&\ge \frac{\big|\big(w(f_1)-w(f_2)\big)^\top \phi(z)\big|}{\sqrt{\lambda+\big(w(f_1)-w(f_2)\big)^\top \Lambda^h \big(w(f_1)-w(f_2)\big)}}\notag\\
&= \frac{\big\|(\Lambda^h)^{1/2}(w(f_1)-w(f_2))\big\|\cdot\big\|(\Lambda^h)^{-1/2}\phi(z)\big\|}{\|w(f)-w(f')\|_{\Lambda^h}}\cdot\frac{\|w(f)-w(f')\|_{\Lambda^h}}{\sqrt{\lambda+\|w(f)-w(f')\|_{\Lambda^h}^2}}\notag\\
&= \frac{\big\|(\Lambda^h)^{1/2}(w(f_1)-w(f_2))\big\|\cdot\big\|(\Lambda^h)^{-1/2}\phi(z)\big\|}{\|w(f)-w(f')\|_{\Lambda^h}}\cdot\frac{1}{\sqrt{\lambda/\|w(f)-w(f')\|_{\Lambda^h}^2+1}}.
\#
Since $\|w(f)-w(f')\|_{\Lambda^h}^2\ge\sqrt{\lambda}$, we have
\$
\sqrt{\frac{\lambda}{\|w(f)-w(f')\|_{\Lambda^h}^2}+1} \le \sqrt{\lambda^{1/2}+1} \le \lambda^{1/4} + 1.
\$
Taking this result back into \eqref{eq:b^hz_lower bound} leads to
\$
\\
b^h(a) &\ge \big\|(\Lambda^h)^{-1/2}\phi(z)\big\| \cdot \frac{1}{\sqrt{\lambda/\|w(f)-w(f')\|_{\Lambda^h}^2+1}}\\
&\ge \frac{\sqrt{\phi(z)^\top (\Lambda^h)^{-1} \phi(z)}}{\lambda^{1/4}+1} = \frac{b_L^h(z)}{\lambda^{1/4}+1}.
\$
Therefore, we conclude the proof.
\end{proof}

\paragraph{Connection between the bootstrapped uncertainty and bonus functions.}
We begin with illustrating the equivalence between the bootstrapped uncertainty and the linear form of the bonus in the following lemma. Since we actually compute the uncertainty weights for only single iteration, we let $\sigma_i^h=1$.
\begin{lemma}\label{lm:variance and bonus}
For any $z\in\cX\times\cA$,
\$
\Var_{\hw^h}(Q^h_{\hw^h}(z)) = \Var_{\hw^h}(z^\top\hw^h) = z^\top (\Lambda^h)^{-1} z.
\$
\end{lemma}
\begin{proof}
Let $y_i^h=r^h(z^h_i) + V_w^{h+1}(z^{h+1}_i)$. Under the assumption that $\epsilon_i^h\sim N(0,1)$, since the closed form solution to the problem is
\$
\hw^h = (\Lambda^h)^{-1} \sum_{z_i^h\in\cD} y_i^h z_i^h= (\Lambda^h)^{-1} \sum_{z_i^h\in\cD} (Q_\phi^h(z_i^h)+\epsilon_i^h) z_i^h,
\$
we obtain that
\$
\hw^h | \cD \sim N(\mu^h, (\Lambda^h)^{-1}),
\$
where 
\$
\mu^h = (\Lambda^h)^{-1} \sum_{z_i^h\in\cD} Q_\phi^h(z_i^h) z_i^h,\quad \Lambda^h=\sum_{i=1}^n z_i^h(z_i^h)^\top.
\$
Then, it follows that for any $z\in\cA\times\cA$,
\$
\Var_{\hw^h}(Q^h_{\hw^h}(z)|\cD) = \Var_{\hw^h}(z^\top\hw^h | \cD) = z^\top (\Lambda^h)^{-1} z.
\$
Hence, we complete the proof.
\end{proof}
In Lemma \ref{lm:variance and bonus}, we find that the standard deviation of the $Q$-value function is equivalent to the linear LCB-bonus $b_L^h(z) = \sqrt{z^\top (\Lambda^h)^{-1} z}$. Moreover, recall that the bootstrapped uncertainty $\mathbb{V}_{j=1,\ldots,N}\left[Q_{w_j}\right]$ is the standard deviation of the bootstrapped $Q$-value functions. Therefore, according to \citet{osband2016deep} our proposed bootstrapped uncertainty can serve as an estimation for the bonus function $b_L^h$ under the linear MDP setting. Theoretically, we can use the uncertainty weight iteration (Algorithm \ref{alg:wi}) to construct the weighted bootstrap uncertainty.

By combining Lemma \ref{lm:equivalence bonus linear and general} and Lemma \ref{lm:variance and bonus}, we conclude that by taking a sufficiently broad function approximation class $\hcF^h$ and sufficiently small parameter $\lambda$, the proposed bootstrapped uncertainty is an estimation of the general form of the bonus in linear MDPs. More importantly, in the experiments, the estimation of the uncertainty is simplified and shares the spirit with the theoretical analysis due to two reasons: 1) due to the complexity of the nonlinear version of uncertainty, which is expressed as
\$
\sup_{f,f'\in\cF}\frac{|f(z_i) - f'(z_i)|}{\sqrt{\lambda + \sum_{j=1}^n(f(z_j) - f'(z_j))^2/\sigma_j^2}},
\$
we simplify the estimation by using the bootstrap uncertainty, which is an unbiased estimation of uncertainty in the linear version, and 2) combining the uncertainty weight iteration and the bootstrap uncertainty estimation is cumbersome, so we only iterate once during simulations.

\section{Results for Distribution Shift}\label{s:ds}
In this section, we consider an MDP$(\cX,\cA,H,\Pb,r)$ and an offline dataset $\cD=\{(x_i^h,a_i^h)\}_{i,h=1}^H$ with adversarial corruption and distribution shift. Specifically, for each trajectory $\{(x_i^h,a_i^h)\}_{h=1}^H$, we define $\rho_i>0$ to measure the ditribution shift of this trajectory. For example, when the learner's goal varies from the training data, the goal shift can be embedded into the initial state $x^1$ and captured by $\rho(x^1)$. Hence, we define a new notion of corruption level $\zeta$, capturing both adversarial corruption and distribution shift.
\begin{definition}[Cumulative Corruption]\label{def:cor_mdp_ds}
The cumulative corruption is $\zeta$ if at any step $h\in[H]$, for any sequence $\{x_i^h,a_i^h\}_{i,h=1}^{n,H}\subset\cX\times\cA$, $\{\rho_i\}_{i=1}^n\subset\rR^+$ and function $\{g^h:\cX\rightarrow[0,1]\}_{h=1}^H$, we have
$$ 
\zeta =\sum_{h=1}^H\sum_{i=1}^n\rho_i|\zeta_i^h|,\quad \zeta_i^h = (\cT^hg-\cT_i^hg)(x_i^h,a_i^h).
$$
\end{definition}

\begin{algorithm}[th]
\caption{Uncertainty Weight Iteration with Distribution Shift}
\label{alg:wi_ds}
\begin{algorithmic}[1]
\STATE {\bf Input:} $\{(x_i, a_i,\rho_i)\}_{i=1}^n,\cF,\alpha>0$
\STATE {\bf Initialization:} $t=0,~\sigma_i^0=1$, $i=1,\ldots,n$
\REPEAT
\STATE $t\leftarrow t+1$
\STATE $(\sigma_i^t)^2 \leftarrow \max\left(1,\sup_{f,f'\in\cF}\frac{|f(x_i,a_i) - f'(x_i,a_i)|/(\alpha\rho_i)}{\sqrt{\lambda + \sum_{j=1}^n(f(x_j,a_j) - f'(x_j,a_j))^2/(\sigma_j^{t-1})^2}}\right)$, $i=1,\ldots,n$
\UNTIL $\max_{i\in[n]} \big(\sigma_i^{t}/\sigma_i^{t-1}\big)^2 \le 2$
\STATE {\bf Output:} $\{\sigma_i^t\}_{i=1}^n$
\end{algorithmic}
\end{algorithm}

The main challenge in handling distribution shifts is the new weight design. We propose uncertainty weight iteration with distribution shift in Algorithm \ref{alg:wi_ds}, where we put $\rho_i$ on the denominator. Similarly, we can follow Lemma \ref{lm:converge_weight} to demonstrate that the iteration converges and the output weights
$\{\sigma_i:=\sigma_i^N\}_{i=1}^n$ satisfy
\#\label{eq:approximate_weight_ds}
\sigma_i^2 \ge \max\Big(1,\frac{1}{2} \sup_{f,f'\in\cF}\frac{|f(x_i,a_i) - f'(x_i,a_i)|/(\alpha\rho_i)}{\sqrt{\lambda + \sum_{j=1}^n(f(x_j,a_j) - f'(x_j,a_j))^2/\sigma_j^2}}\Big).
\#

Therefore, just by replacing the weight iteration (Algorithm \ref{alg:wi}) in CR-PEVI Algorithm \ref{alg:mdp} with Algorithm \ref{alg:wi_ds}, we obtain an algorithm robust to both corruption and distribution shift, named as \textbf{CORDS-PEVI}. Because CORDS-PEVI highly repeats Algorithm \ref{alg:mdp}, we do not present the pseudo-code of the algorithm. 

Then, the suboptimality bound achieved by CORDS-PEVI is presented in the following theorem.
\begin{theorem}\label{th:mdp_ds}
If the coverage coefficient in Definition \ref{df:coverage_condition_mdp} is finite and Assumption \ref{as:Well-Explored Dataset} holds, under CORDS-PEVI, for any cumulative corruption $\zeta$ and $\delta>0$, we choose the covering parameter $\gamma=1/(n\max_h\beta^h\zeta^h)$, the eluder parameter $\lambda=\ln(N_n(\gamma))$, the weighting parameter $\alpha=H\sqrt{\ln N_n(\gamma)}/\zeta$, and the confidence radius
$$
\beta^h=c_{\beta}\left( \alpha\zeta^h + \sqrt{\ln(HN_n(\gamma)/\delta)}\right)~\text{for}~h=H,\ldots,1,
$$
where
$$
N_n(\gamma)=\max_h N\Big(\frac{\gamma}{n},\cF^h\Big)\cdot N\Big(\frac{\gamma}{n},\cF^{h+1}\Big)\cdot N\Big(\frac{\gamma}{n},\cB^{h+1}(\lambda)\Big).
$$
Then, with probability at least $1-3\delta$, the sub-optimality is bounded by
$$
\subopt(\hat\pi, x) = \tcO\bigg(\frac{H(\CC(\lambda,\hcF,\cZ_n^H))^{1/4}\cdot(\ln N_n(\gamma))^{1/2}}{n^{1/2}(C(\hcF,\mu))^{1/4}} + \frac{\zeta(\CC(\lambda,\hcF,\cZ_n^H))^{1/2}}{n(C(\hcF,\mu))^{1/2}}\bigg).
$$
\end{theorem}

\subsection{Analysis of the Result}
The main difference in the analysis between the model with and without a distribution shift is the bound for the confidence radius.

\subsubsection{Sharp bound of the confidence radius.}

\begin{lemma}[Confidence Radius]\label{lm:Confidence_Radius_mdp_ds}
In CORDS-PEVI, for all $h\in[H]$ we have $\cT^hf_n^{h+1}\in\hcF^h$ with probability at least $1-\delta$, where
$$
\beta^h=12\alpha\zeta^h  + \big(12\lambda + 12\ln(2HN_n^h(\gamma)/\delta) + 12(5\beta^{h+1}\gamma)^2n + 60\beta^{h+1}\gamma\sqrt{nC_1^h(n,\zeta)}\big)^{1/2},
$$
we use the notation
$$
\begin{aligned}
N_n(\gamma)=\max_h N\Big(\frac{\gamma}{n},\cF^h\Big)\cdot N\Big(\frac{\gamma}{n},\cF^{h+1}\Big)\cdot N\Big(\frac{\gamma}{n},\cB^{h+1}(\lambda)\Big),
\end{aligned}
$$
and $C_1^h(n,\zeta)=2(\sum_{i=1}^n(\zeta_i^h)^2 + 2n\eta^2 + 3\eta^2\ln(2/\delta))$, $\zeta^h=\sum_{i=1}^n\rho_i\zeta_i^h$.
\end{lemma}
\begin{proof}
We use similar methods as the proof of Lemma \ref{lm:Confidence_Radius_mdp}. At each step $h\in[H]$, by notating $\cF_{\gamma}^{h+1}$ as a $(\gamma,\|\cdot\|_{\infty})$ cover of $\cF^{h+1}$, and $\cB^{h+1}_{\gamma}$ as a $(\gamma,\|\cdot\|_{\infty})$ cover of $\cB^{h+1}(\lambda)$, we construct $\Bar{\cF}_{\gamma}^{h+1}=\cF_{\gamma}^{h+1}\oplus\beta_{\tau}^{h+1}\cB^{h+1}_{\gamma}$ as a $((1+\beta^{h+1})\gamma,\|\cdot\|_{\infty})$ cover of $f_n^{h+1}(\cdot)$. For the $f_n^{h+1}$, there exists a $\Bar{f}^{h+1}\in\Bar{\cF}_{\gamma}^{h+1}$ such that $\|\Bar{f}^{h+1}-f_n^{h+1}\|_{\infty}\le\beps=(1+\beta^{h+1})\gamma$. Then, we define $y_i^h = r_i^h + f_n^{h+1}(x_i^{h+1})$, $\Bar{y}_i^h=r_i^h+\Bar{f}^{h+1}(x_i^{h+1})$ and
$$
\tilde{f}^h = \argmin_{f^h\in\cF^h}\sum_{i=1}^n (f^h(x_i^h,a_i^h)-\Bar{y}_i^h)^2.
$$
Since \eqref{eq:approximate_erm} also holds true, we can invoke Lemma \ref{lm:empirical_diff_mdp} by taking $f_*$ as $\E[\Bar{y}_i^h|x_i^h,a_i^h]=(\cT_\cD^h\Bar{f}^{h+1})(x_i^h,a_i^h)$ and $f_b$ as $\cT^h\Bar{f}^{h+1}$. With probability at least $1-\delta$, we obtain:
\#\label{eq:sum_hat_fkh_Tb_barf0_ds}
&\sum_{i=1}^n \frac{\left(\hat{f}^h(x_i^h,a_i^h)-(\cT^h\Bar{f}^{h+1})(x_i^h,a_i^h)\right)^2}{(\sigma_i^h)^2}\nonumber\\
&\qquad\le 10\ln(2HN_n^h(\gamma)/\delta) + 5\underbrace{\sum_{i=1}^n \frac{|\hat{f}^h(x_i^h,a_i^h)-(\cT^h\Bar{f}^{h+1})(x_i^h,a_i^h)|\cdot|\zeta_i^h|}{(\sigma_i^h)^2}}_{(a)}\nonumber\\
&\qquad + 10(\gamma + 2\beps)\cdot\left((\gamma + 2\beps)n + \sqrt{nC_1(n,\zeta)}\right),
\#
where $C_1(n,\zeta)=2(\zeta^2 + 2n + 3\ln(2/\delta))$. 

From the weight design and Lemma \ref{lm:converge_weight}, term (a) is bounded by
\$
\sum_{i=1}^n|\zeta_i^h|\cdot\frac{|\hat{f}^h(x_i^h,a_i^h)-(\cT^h\Bar{f}^{h+1})(x_i^h,a_i^h)|}{(\sigma_i^h)^2} 
&\le \sum_{i=1}^n \rho_i|\zeta_i^h|\cdot\frac{|\hat{f}^h(x_i^h,a_i^h)-(\cT^h\Bar{f}^{h+1})(x_i^h,a_i^h)|}{\rho_i(\sigma_i^h)^2}\\
&\le \zeta^h \sup_i \frac{|\hat{f}^h(x_i^h,a_i^h)-(\cT^h\Bar{f}^{h+1})(x_i^h,a_i^h)|}{\rho_i(\sigma_i^h)^2}\\
&\le 2\alpha\zeta^h\sqrt{\lambda + \sum_{i=1}^n \frac{\left(\hat{f}^h(x_i^h,a_i^h)-(\cT^h\Bar{f}^{h+1})(x_i^h,a_i^h)\right)^2}{(\sigma_i^h)^2}}
\$
where the last inequality is obtained since $\sigma_i^h$ satisfies \eqref{eq:approximate_weight_ds}.
Taking this result back into \eqref{eq:sum_hat_fkh_Tb_barf0_ds}, we finally get for all $h\in[H]$,
\$
&\sum_{i=1}^n \frac{\big(\hat{f}^h(x_i^h,a_i^h)-(\cT^h\Bar{f}^{h+1})(x_i^h,a_i^h)\big)^2}{(\sigma_i^h)^2}\\
&\qquad\le 10\ln(2HN_n^h(\gamma)/\delta) + 10\alpha\zeta^h\beta^h + 5\gamma\zeta + 10(\gamma + 2\beps)\cdot((\gamma + 2\beps)n + \sqrt{nC_1(t,\zeta)})\\
&\qquad= 10\ln(2HN_n^h(\gamma)/\delta) + 10\alpha\zeta^h\beta^h + 5\gamma\zeta + 10(2\beta^{h+1}+3)^2\gamma^2n + 10(2\beta^{h+1}+3)\gamma\sqrt{nC_1(n,\zeta)},
\$
where the last euqlaity uses $\beps=(1+\beta^{h+1})\gamma$.
Therefore, it follows that with probability at least $1-\delta$,
\$
&\left(\sum_{i=1}^{n}\frac{\left(\hat{f}_n^h(x_i^h,a_i^h)-(\cT^h f_n^{h+1})(x_i^h,a_i^h)\right)^2}{(\sigma_i^h)^2} + \lambda\right)^{1/2}\\
&\qquad\le \left(\sum_{i=1}^{n}\frac{\left(\hat{f}_n^h(x_i^h,a_i^h)-(\cT^h\Bar{f}_n^{h+1})(x_i^h,a_i^h)\right)^2}{(\sigma_i^h)^2}\right)^{1/2} + \sqrt{n}\beps + \sqrt{\lambda}\\
&\qquad\le \left(10\ln(2HN_n^h(\gamma)/\delta) + 10\alpha\zeta^h\beta^h + 5\gamma\zeta + 10(2\beta^{h+1}+3)^2\gamma^2n \right.\\
&\qquad\qquad \left. + 10(2\beta^{h+1}+3)\gamma\sqrt{nC_1(n,\zeta)}\right)^{1/2} + (\beta^{h+1}+1)\gamma\sqrt{n} + \sqrt{\lambda}\\
&\qquad\le \left(12\lambda + 12\ln(2HN_n^h(\gamma)/\delta) + 24\alpha\zeta^h\beta^h + 12(5\beta^{h+1}\gamma)^2n + 60\beta^{h+1}\gamma\sqrt{nC_1(n,\zeta)}\right)^{1/2}\\
&\qquad\le \beta^h,
\$
which finishes the proof.
\end{proof}

\subsubsection{Connections between weighted and unweighted coefficient.}

We define the coverage coefficient incorporating the uncertainty weights under distribution shift as
\#\label{eq:coverage coefficient weighted_ds}
\CC^{\sigma}(\lambda,\hcF,\cZ_n^H) = \max_{h\in[H]}\E_{\pi_*}\bigg[\sup_{f,f'\in\hcF^h}\frac{n(f(x^h,a^h) - f'(x^h,a^h))^2/\sigma^h(x^h,a^h)^2}{\lambda + \sum_{i=1}^n(f(x_i^h,a_i^h) - f'(x_i^h,a_i^h))^2/(\sigma_i^h)^2} \,\bigg|\, x^1=x\bigg],
\#
where the weight for the trajectory induced by the optimal policy is
\#\label{eq:sigma^h(x^h,a^h)^2_ds}
(\sigma_*^h(x^h,a^h))^2 = \max\bigg(1, \sup_{f,f'\in\hcF^h} \frac{|f(x^h,a^h) - f'(x^h,a^h)|/\alpha}{\sqrt{\lambda + \sum_{i=1}^n(f(x_i^h,a_i^h) - f'(x_i^h,a_i^h))^2/(\sigma_i^h)^2}}\bigg),
\#
where we do not consider the distribution shift for the expected trajectory induced by the optimal policy $\pi_*$.

\begin{lemma}\label{lm:Connections between Weighted and Unweighted Coefficient_ds}
Let $\rho_{\min}=\min_{i\in[n]}\rho_i$. Under CORDS-PEVI, Assumption \ref{as:Well-Explored Dataset}, and the $\beta^h=C_{\beta}\sqrt{\ln N}$ (where $C_{\beta}>0$ contains the logarithmic terms that can be omitted) and $\lambda$ given in Theorem \ref{th:mdp_ds}, we have
\$
\CC^{\sigma}(\lambda,\hcF,\cZ_n^H) = \tcO\big((\CC(\lambda,\hcF,\cZ_n^H))^{1/2}\cdot (C(\hcF,\mu))^{-1/2}\big).
\$
\end{lemma}
\begin{proof}
We adopt the same approaches in the proof of Lemma \ref{lm:Connections between Weighted and Unweighted Coefficient}. For each $i\in[n],h\in[H]$, since the weight $(\sigma_i^h)^2$ yielded by Algorithm \ref{alg:wi_ds} is upper bounded by $1/(\alpha\sqrt{\lambda}\rho_{\min})$, we get
\#\label{eq:CC_bound1_ds}
\CC^{\sigma}(\lambda,\hcF,\cZ_n^H) &= \max_{h\in[H]} \sum_{z^h} d_{\pi_*}^h(z^h) \sup_{f,f'\in\hcF^h}\frac{n(f(z^h) - f'(z^h))^2/(\sigma_*^h(z^h))^2}{\lambda + n\sum_{\bar z^h} \mu(z^h)(f(\bar z^h) - f'(\bar z^h))^2/(\sigma^h(\bar z^h))^2}\notag\\
&\le \max_{h\in[H]} \sum_{z^h} d_{\pi_*}^h(z^h) \sup_{f,f'\in\hcF^h}\frac{(f(z^h) - f'(z^h))^2/(\sigma_*^h(z^h))^2}{\lambda/n + \alpha\sqrt{\lambda}\sum_{\bar z^h} \mu(\bar z^h)(f(\bar z^h) - f'(\bar z^h))^2}\notag\\
&\le \frac{1}{\alpha\rho_{\min}\sqrt{\lambda}}\max_{h\in[H]} \sum_{z^h} d_{\pi_*}^h(z^h) \sup_{f,f'\in\hcF^h}\frac{(f(z^h) - f'(z^h))^2/(\sigma_*^h(z^h))^2}{\lambda/n + \sum_{\bar z^h} \mu(\bar z^h)(f(\bar z^h) - f'(\bar z^h))^2},
\#
where the last inequality uses $\alpha\rho_{\min}\sqrt{\lambda}\le 1$. For any $z^h\sim d_{\pi_*}^h(z^h)$, take the $f_{z^h},f_{z^h}'\in\hcF^h$ that maximize the term:
\#\label{eq:f,f'max_term_ds}
\frac{(f_{z^h}(z^h) - f_{z^h}'(z^h))^2/(\sigma_*^h(z^h))^2}{\lambda/n + \sum_{\bar z^h} \mu(\bar z^h)(f_{z^h}(\bar z^h) - f_{z^h}'(\bar z^h))^2}.
\#
Then, \eqref{eq:CC_bound1_ds} is written as
\#\label{eq:CC_bound2_ds}
\CC^{\sigma}(\lambda,\hcF,\cZ_n^H) &\le \frac{1}{\alpha\rho_{\min}\sqrt{\lambda}}\max_{h\in[H]} \sum_{z^h} d_{\pi_*}^h(z^h) \frac{(f_{z^h}(z^h) - f_{z^h}'(z^h))^2/(\sigma_*^h(z^h))^2}{\lambda/n + \sum_{\bar z^h} \mu(\bar z^h)(f_{z^h}(\bar z^h) - f_{z^h}'(\bar z^h))^2}\notag\\
&= \frac{1}{\alpha\rho_{\min}\sqrt{\lambda}}\max_{h\in[H]} \sum_{z^h} d_{\pi_*}^h(z^h) \frac{(f_{z^h}(z^h) - f_{z^h}'(z^h))^2/(\sigma_*^h(z^h))^2}{\lambda/n + \sum_{\bar z^h} \mu(\bar z^h)(f_{z^h}(\bar z^h) - f_{z^h}'(\bar z^h))^2}.
\#
To bound the above term, we need to lower bound $(\sigma_*^h(z^h))^2$ in \eqref{eq:sigma^h(x^h,a^h)^2_ds} by
\$
\sup_{f,f'\in\hcF^h} \frac{|f(z^h) - f'(z^h)|/\alpha}{\sqrt{\lambda + \sum_{i=1}^n(f(z_i^h) - f'(z_i^h))^2/(\sigma_i^h)^2}} &\ge \frac{|f_{z^h}(z^h) - f_{z^h}'(z^h)|/\alpha}{\sqrt{\lambda + \sum_{i=1}^n(f_{z^h}(z_i^h) - f_{z^h}'(z_i^h))^2/(\sigma_i^h)^2}}\\
&\ge \frac{|f_{z^h}(z^h) - f_{z^h}'(z^h)|}{2\alpha\beta^h},
\$
where the last inequality uses $f_{z^h},f_{z^h}\in\hcF^h$:
\$
\lambda + \sum_{i=1}^n\frac{(f_{z^h}(z_i^h) - f_{z^h}'(z_i^h))^2}{(\sigma_i^h)^2} &= \lambda + \sum_{i=1}^n\frac{(f_{z^h}(z_i^h) - \hf(z_i^h) + \hf(z_i^h) - f_{z^h}'(z_i^h))^2}{(\sigma_i^h)^2}\\
&\le 2\Big[\lambda + \sum_{i=1}^n\frac{(f_{z^h}(z_i^h) - \hf(z_i^h))^2}{(\sigma_i^h)^2} + \lambda + \sum_{i=1}^n\frac{\hf(z_i^h) - f_{z^h}'(z_i^h))^2}{(\sigma_i^h)^2}\Big] \le 4(\beta^h)^2.
\$
Thus, substituting this lower bound into \eqref{eq:CC_bound2_ds} and taking $\beta^h=C_{\beta}\sqrt{\ln N}=C_{\beta}\sqrt{\lambda}$, we get
\$
\CC^{\sigma}(\lambda,\hcF,\cZ_n^H) &\le \frac{2\alpha\beta^h}{\alpha\rho_{\min}\sqrt{\lambda}}\max_{h\in[H]} \sum_{z^h} d_{\pi_*}^h(z^h) \frac{|f_{z^h}(z^h) - f_{z^h}'(z^h)|}{\lambda/n + \sum_{\bar z^h} \mu^h(\bar z^h)(f_{z^h}(\bar z^h) - f_{z^h}'(\bar z^h))^2}\notag\\
&= \frac{2C_{\beta}}{\rho_{\min}} \max_{h\in[H]} \sum_{z^h} d_{\pi_*}^h(z^h) \frac{|f_{z^h}(z^h) - f_{z^h}'(z^h)| \cdot \mathbbm 1\big(\|f_{z^h}-f_{z^h}'\|_{\infty}\le n^{-1}\big)}{\lambda/n + \sum_{\bar z^h} \mu^h(\bar z^h)(f_{z^h}(\bar z^h) - f_{z^h}'(\bar z^h))^2}\notag\\
&\qquad + \frac{2C_{\beta}}{\rho_{\min}} \max_{h\in[H]} \sum_{z^h} d_{\pi_*}^h(z^h) \frac{|f_{z^h}(z^h) - f_{z^h}'(z^h)|\cdot \mathbbm 1\big(\|f_{z^h}-f_{z^h}'\|_{\infty}> n^{-1}\big)}{\lambda/n + \sum_{\bar z^h} \mu^h(\bar z^h)(f_{z^h}(\bar z^h) - f_{z^h}'(\bar z^h))^2}\\
&\le \frac{2C_{\beta}}{\lambda\rho_{\min}} + \frac{2C_{\beta}}{\rho_{\min}} \max_{h\in[H]} \sum_{z^h} d_{\pi_*}^h(z^h) \frac{|f_{z^h}(z^h) - f_{z^h}'(z^h)|\cdot \mathbbm 1\big(\|f_{z^h}-f_{z^h}'\|_{\infty}> n^{-1}\big)}{\lambda/n + \sum_{\bar z^h} \mu^h(\bar z^h)(f_{z^h}(\bar z^h) - f_{z^h}'(\bar z^h))^2}.
\$
Then, we can invoke the Cauchy-Schwarz inequality to split the above term  into two parts:
\#\label{eq:CC_bound4_ds}
\CC^{\sigma}(\lambda,\hcF,\cZ_n^H) &\le \frac{2C_{\beta}}{\lambda\rho_{\min}} + \frac{2C_{\beta}}{\rho_{\min}}\max_{h\in[H]} \sum_{z^h} \frac{\sqrt{d_{\pi_*}^h(z^h)}}{\sqrt{\lambda/n + \sum_{\bar z^h} \mu^h(\bar z^h)(f_{z^h}(\bar z^h) - f_{z^h}'(\bar z^h))^2}}\notag\\
&\qquad\cdot \frac{\sqrt{d_{\pi_*}^h(z^h)}|f_{z^h}(z^h) - f_{z^h}'(z^h)|}{\sqrt{\lambda/n + \sum_{\bar z^h} \mu^h(\bar z^h)(f_{z^h}(\bar z^h) - f_{z^h}'(\bar z^h))^2}} \cdot \mathbbm 1\big(\|f_{z^h}-f_{z^h}'\|_{\infty}> n^{-1}\big)\notag\\
&\le \frac{2C_{\beta}}{\lambda\rho_{\min}} + \frac{2C_{\beta}}{\rho_{\min}}\max_{h\in[H]} \Big(\underbrace{\sum_{z^h} \frac{d_{\pi_*}^h(z^h)\cdot \mathbbm 1\big(\|f_{z^h}-f_{z^h}'\|_{\infty}> n^{-1}\big)}{\lambda/n + \sum_{\bar z^h} \mu^h(\bar z^h)(f_{z^h}(\bar z^h) - f_{z^h}'(\bar z^h))^2}}_{I_1}\Big)^{1/2} \notag\\
&\qquad\cdot \Big(\underbrace{\max_{h\in[H]}\sum_{z^h} d_{\pi_*}^h(z^h)\frac{(f_{z^h}(z^h) - f_{z^h}'(z^h))^2}{\lambda/n + \sum_{\bar z^h} \mu^h(\bar z^h)(f_{z^h}(\bar z^h) - f_{z^h}'(\bar z^h))^2}}_{I_2}\Big)^{1/2},
\#
Then, the terms $I_1$ and $I_2$ can be handled in the same way as the proof of Lemma \ref{lm:Connections between Weighted and Unweighted Coefficient}. For the term $I_1$, we get
\$
I_1 \le \sum_{z^h} \frac{d_{\pi_*}^h(z^h)}{C(\hcF,\mu)/2} = \frac{2}{C(\hcF,\mu)}.
\$
For the term $I_2$, we have
\$
I_2 &= \max_{h\in[H]} \sum_{z^h} d_{\pi_*}^h(z^h) \sup_{f,f'\in\hcF^h}\frac{n(f(z^h) - f'(z^h))^2}{\lambda + \sum_{i=1}^n (f(z_i^h) - f'(z_i^h))^2} = \CC(\lambda,\hcF,\cZ_n^H).
\$
Therefore, by taking the bound of the terms $I_1,I_2$ into \eqref{eq:CC_bound4_ds}, we have
\$
\CC^{\sigma}(\lambda,\hcF,\cZ_n^H) = \frac{2C_{\beta}}{\lambda\rho_{\min}} + \frac{2C_{\beta}}{\rho_{\min}}\sqrt{\frac{2\CC(\lambda,\hcF,\cZ_n^H)}{C(\hcF,\mu)}} = \tcO\Big(\sqrt{\frac{\CC(\lambda,\hcF,\cZ_n^H)}{C(\hcF,\mu)}}\Big),
\$
which concludes the proof.
\end{proof}

\subsubsection{The suboptimality bound.}
Having the guarantee for the confidence radius and the connection between weighted and unweighted coverage coefficient, we can follow the three steps in the proof of Theorem \ref{th:mdp} to bound the suboptimality.
\begin{proof}[Proof of Theorem \ref{th:mdp_ds}]
Since when $\sigma_*^h(x^h,a^h)>1$, we have
\$
\sigma_*^h(x^h,a^h)&=\frac{1}{\sigma_*^h(x^h,a^h)}\cdot\sup_{f,f'\in\hcF^h} \frac{|f(x^h,a^h) - f'(x^h,a^h)|/\alpha}{\sqrt{\lambda + \sum_{i=1}^n(f(x_i^h,a_i^h) - f'(x_i^h,a_i^h))^2/(\sigma_i^h)^2}}= \frac{b^h(x^h,a^h)}{\alpha\sigma_*^h(x^h,a^h)},
\$
from which it follows that
\$
&\sum_{h=1}^H\beta^h\E_{\tilde{\pi}_*}\big[ b^h(x^h,a^h)\,\big|\, x^1=x\big]\\
&\qquad \le \sum_{h=1}^H \beta^h \E_{\tilde{\pi}_*} \left[ \frac{b^h(x^h,a^h)}{\sigma_*^h(x^h,a^h)} + \left(\frac{b^h(x^h,a^h)}{\sigma_*^h(x^h,a^h)}\right)^2 \cdot \frac{1}{\alpha} \,\bigg|\, x^1=x \right]\\
&\qquad \le \sum_{h=1}^H \left[\beta^h \cdot \sqrt{\frac{\CC^{\sigma}(\lambda,\hcF,\cZ_n^H)}{n}} + \frac{\beta^h}{\alpha} \cdot \frac{\CC^{\sigma}(\lambda,\hcF,\cZ_n^H)}{n}\right],
\$
where the last inequality uses $\E X\le \sqrt{\E X^2}$ and
\$
\E_{\pi_*} \bigg[\Big(\frac{b^h(x^h,a^h)}{\sigma_*^h(x^h,a^h)}\Big)^2 \,\bigg|\, x^1=x \bigg] = \frac{\CC^{\sigma}(\lambda,\hcF,\cZ_n^H)}{n}.
\$
We further get with probability at least $1-\delta$,
\$
&\sqrt{\frac{\CC^{\sigma}(\lambda,\hcF,\cZ_n^H)}{n}} \sum_{h=1}^H \beta^h + \frac{\CC^{\sigma}(\lambda,\hcF,\cZ_n^H)}{n}\cdot\frac{\sum_{h=1}^H\beta^h}{\alpha}\\
&\qquad \le \sqrt{\frac{\CC^{\sigma}(\lambda,\hcF,\cZ_n^H)}{n}} \cdot c_{\beta}\big(\alpha\zeta  + H\sqrt{\ln N}\big) + \frac{\CC^{\sigma}(\lambda,\hcF,\cZ_n^H)}{n}\cdot c_{\beta}\Big(\zeta  + \frac{H\sqrt{\ln N}}{\alpha}\Big)\\
&\qquad = \tcO\bigg(\frac{3H\sqrt{\CC^{\sigma}(\lambda,\hcF,\cZ_n^H)\cdot\ln N}}{\sqrt{n}} + \frac{3\zeta\cdot\CC^{\sigma}(\lambda,\hcF,\cZ_n^H)}{n}\bigg)
\$
where we choose $\alpha=H\sqrt{\ln N}/\zeta$. 
By invoking Lemma \ref{lm:Connections between Weighted and Unweighted Coefficient_ds}, we obtain that with probability at least $1-2\delta$,
\$
\subopt(\hat\pi, x) = \tcO\bigg(\frac{H(\CC(\lambda,\hcF,\cZ_n^H))^{1/4}\cdot(\ln N_n(\gamma))^{1/2}}{n^{1/2}(C(\hcF,\mu))^{1/4}} + \frac{\zeta(\CC(\lambda,\hcF,\cZ_n^H))^{1/2}}{n(C(\hcF,\mu))^{1/2}}\bigg).
\$

Ultimately, we conclude the proof.
\end{proof}

\section{Implementation Details}
\label{ap:implement_detail}

\paragraph{Implementation of UWMSG.} Following prior uncertainty-based offline RL algorithm Model Standard-deviation Gradients (MSG) \cite{ghasemipour2022so}, we learn a group of $Q$ networks $Q_{w_i}, i=1,\ldots,K$ with independent targets and optimize a policy $\pi_{\theta}$ with a lower-confidence bound (LCB) objective \cite{ghasemipour2022so,bai2022pessimistic}. Specifically, $Q_{w_i}$ is learned to minimize the following weighted regression objective with samples from the offline dataset $\mathcal{D}$:
\begin{equation}
\begin{aligned}
\min_{w_i}  \mathbb{E}_{(x,a,r,x')\sim \mathcal{D}} \Big[ & \frac{\big(\widehat{\mathcal{T}} Q_{w_i}(x,a) - Q_{w_i}(x,a)\big)^2}{\sigma(x,a)^2} \Big].
\label{eq:Q_obj}
\end{aligned}
\end{equation}

The weight function $\sigma$ is estimated via bootstrapped uncertainty:
$
    % \sigma(x,a) = \max(\sqrt{\mathbb{V}_{i=1,\ldots,K}\left[Q_{w_i}(x,a)\right]}, 1)
    \sigma(x,a)=\text{clip}( \mu \times  \sqrt{\mathbb{V}_{i=1,\ldots,K}\left[Q_{w_i}(x,a)\right]}, 1, M)
$, where $\mathbb{V}_{i=1,\ldots,K}\left[Q_{w_i}\right]$ refers to the variance between the group of $Q$ functions, and $M$ is used to control the maximum value of the weighting function. Note that $\sigma(x,a)$ is detached from the gradients, and the update is exclusively on $w_i$. We introduce the uncertainty ratio $\mu$ for $\sigma(x,a)$ to tune the weight function.
The independent target $\widehat{\mathcal{T}}Q_{w_i}$ for $Q_{w_i}$ is defined as follows:
\begin{equation}
\label{eq:q_target}
    \widehat{\mathcal{T}}Q_{w_i}(x, a):=r(x,a)+\gamma {\mathbb{E}}_{a'\sim \pi_\theta(\cdot|x')}\big[ Q_{w'_i}(x',a')\big],
\end{equation}
where $Q_{w'_i}$ is the target network for $Q_{w_i}$. In empirical offline RL, it is a common practice \cite{an2021uncertainty,bai2022pessimistic,ghasemipour2022so} to utilize the discounted form of the Q function rather than the episodic version. The policy $\pi_{\theta}$ optimizes the same pessimistic objective as MSG:
\begin{equation}
\label{eq:policy_obj}
\min_{\theta} \mathbb{E}_{x \sim \mathcal{D}, a \sim \pi_{\theta}(\cdot|x)}  \Big[\mathbb{E}_{i=1,\ldots,K} \left[Q_{w_i}(x,a)\right] - \beta \cdot \sqrt{\mathbb{V}_{i=1,\ldots,K} \left[Q_{w_i}(x,a) \right]} \Big],
\end{equation}
Although the weighting function $\sigma(x,a)$ shares some similarities with the pessimistic bonus, they differ in the following aspects: (1) \textbf{ $\sigma(x,a)$ measures the intrinsic variance of the corrupted data, while the pessimistic bonus penalizes out-of-distribution (OOD) actions produced by $\pi_{\theta}$; (2) $\sigma(x,a)$ weights the $Q$ learning objective and is detached from gradients, whereas the pessimistic bonus requires gradients for $\pi_{\theta}$}.

\paragraph{Training and Evaluation Details.} We use 3-layer MLPs with 256 neurons in each layer for both $Q$ and policy networks. The ensemble size is set to $K=10$ for all the experiments. The hyperparameters, such as learning rate and optimizer, are listed in Table \ref{tab:hyper-table}. We train each algorithm for 3000 epochs, where one epoch contains 1000 updates. Regarding the offline datasets, we use `halfcheetah-medium-v2', `walker2d-medium-replay-v2', and `hopper-medium-replay-v2' datasets and refer to them as `halfcheetah', `walker2d', and `hopper' in our paper. Our implementation is based on SAC-N \cite{an2021uncertainty}. Therefore, there are additional entropy regularization terms for both $\widehat{\mathcal{T}}Q_{w_i}(s, a)$ in Eq \eqref{eq:q_target} and the policy objective in Eq \eqref{eq:policy_obj}. For hyperparameters related to corruption and uncertainty weighting, we list them in Table \ref{tab:hyper-attack}. Since tasks vary in their ability to resist corruption, their hyperparameters are tuned separately. We use the same LCB ratio $\beta$ for MSG and UWMSG, which is searched within $\{4.0,6.0\}$. The uncertainty ratio $\mu$ for UWMSG is search from $\{0.2,0.3,0.5,0.7,1.0\}$.

To evaluate algorithms, we run the deterministic policy of each agent for 1000 steps and report their average cumulative returns with standard deviations over $10$ random seeds. Our code is based on \cite{tarasov2022corl} and is available at \href{https://github.com/YangRui2015/UWMSG}{https://github.com/YangRui2015/UWMSG}.

\paragraph{Data Corruption Details.}
We implement both random and adversarial corruption on either rewards or dynamics. The four types of data corruption are listed below:
\begin{itemize}[leftmargin=*]
    \item Random reward attack: randomly sample $c\%$ transitions $(x,a,r,x')$ from $D$, and modify the reward $\hat r  \sim \text{Uniform}[-\epsilon, \epsilon]$, where $c$ is the corruption rate and $\epsilon$ is the corruption scale.
    \item  Random dynamics attack: randomly sample $c\%$ transitions $(x,a,r,x')$, and modify the next-step state $\hat x' = x'+\delta \cdot \text{std} , \delta \sim \text{Uniform}[-\epsilon,\epsilon]^{d}$, where $d$ is dimension of states and $\text{std}$ is the $d$-dimensional standard deviation of all states in the offline dataset.
    \item  Adversarial reward attack: randomly sample $c\%$ transitions $(x,a,r,x')$, and modify the reward as: $\hat r=- \epsilon \times r$.
    \item Adversarial dynamics attack: pretrain a group of $Q_{p}$ functions and a policy function $\pi_{p}$, then randomly sample $c\%$ transitions $(x,a,r,x')$, and modify the next-step states $\hat x'=\min_{\hat x'\in \mathbb{B}_d(x',\epsilon)} Q_{p}(x', \pi_{p}(\hat x'))$, where $\mathbb{B}_d(x',\epsilon)=\{|\hat x' - x'| \leq \epsilon \cdot \text{std} \}$ regularizes the maximum difference for each state dimension. The optimization is implemented through gradient descent similar to prior works \cite{zhang2020robust,yangrorl}. 
\end{itemize}

For the implementation of an adversarial dynamics attack, the optimization is performed through 10-step gradient descent with learning rate $\frac{\epsilon}{10}$. After each gradient descent step, the states are clipped within $\mathbb{B}_d(x',\epsilon)$. The pretraining algorithm used is MSG for the halfcheetah and walker2d tasks, while EDAC is employed for the hopper task due to its significantly better performance on this task compared to MSG in the absence of corruption. Finally, the corrupted data is saved and will be loaded for future training. To control the cumulative corruption $\zeta$ under continuous state-action spaces, we incorporate random or adversarial noise with predefined corruption ranges and corruption scales into the rewards and next-step states. This is because the fact that $\zeta_i \leq |r-r_{\mathcal{D}}| + D_{TV}(P\|P_{\mathcal{D}}) \leq |r-r_{\mathcal{D}}| + \sqrt{\frac{1}{2} D_{KL}(P\|P_{\mathcal{D}}})$. When we consider $P$ and $P_{\mathcal{D}}$ are both Diagonal Gaussian distributions with the same constant variance, $\zeta_i \leq  |r-r_{\mathcal{D}}| + \sqrt{\frac{1}{2}} \|\mu - \mu_{\mathcal{D}}\| + \text{const} $. When we corrupt only one element in rewards and dynamics, the empirical cumulative corruption can be approximated as the multiplication of the number of corrupted samples and the corruption scale: $\zeta=|D| \times c\% \times \epsilon$, where $|D|$ represents the size of the dataset, and $\epsilon$ represents the corruption scale. Note that this approximation may not hold for our reward corruption, but we use the same calculation for simplicity. 

 \begin{table}[h!]
\small
  \caption{Hyper-parameters for UWMSG and MSG.}
  \vspace{3pt}
  \label{tab:hyper-table}
  \centering
  \begin{tabular}{p{0.55\columnwidth}p{0.4\columnwidth}}
    \toprule
    Hyper-parameters & Value\\
    \midrule
    Ensemble size $K$          & 10  \\
    Policy network  & FC(256,256,256) with ReLU \\
    $Q$-network  & FC(256,256,256) with ReLU \\
    LCB ratio $\beta$  & $\{4.0, 6.0\}$ \\
    Uncertainty ratio $\mu$ & $\{0.2,0.3,0.5,0.7,1.0\}$ \\
    Maximum value of uncertainty weight $M$ & 10 \\
    Target network smoothing coefficient $\tau$ & 5e-3 \\
    Discount factor $\gamma$ & 0.99 \\
    Policy learning rate & 3e-4 \\
    $Q$ network learning rate & 3e-4 \\
	Optimizer & Adam  \\
	Automatic Entropy Tuning & True \\
	batch size & 256 \\
     \bottomrule
  \end{tabular}
\end{table}

\begin{table}[h!]
\small
  \caption{Data corruption settings and the hyperparameters used for uncertainty-weighting in UWMSG.}
  \vspace{3pt}
  \label{tab:hyper-attack}
  \centering
  \begin{adjustbox}{width=1\columnwidth}
  \begin{tabular}{l|l|c|c|c|c|c|c}
    \toprule
   Attack type & Attack object &  Environment  & Corruption rate $c\%$ & Corruption scale $\epsilon$ & Cumulative corruption $\zeta$ & LCB ratio $\beta$ & Uncertainty ratio $\mu$ \\
    \midrule
   \multirow{6}{*}{Random}  & \multirow{3}{*}{Reward}      & halfcheetah &  20$\%$ & 30.0 & 5.99$\times 10^6$ & 4.0 & 0.7 \\
    &  &  walker2d &  30$\%$ & 30.0 & 2.72$\times 10^6$ & 4.0 & 0.3 \\
    &  &  hopper & 20$\%$ & 30.0 &  2.41$\times 10^6$ & 6.0 & 0.7 \\
    \cline{2-8}
    & \multirow{3}{*}{Dynamics} & halfcheetah & 20$\%$  & 2.0 & 4.00 $\times 10^5$ & 4.0 & 0.5 \\
    &  &  walker2d & 10$\%$ & 0.5 & 1.51$\times 10^4$ & 6.0 & 0.5 \\
    & & hopper & 10$\%$ & 0.5 & 2.01$\times 10^4$ & 6.0  & 0.7  \\
   \hline
 \multirow{6}{*}{Adversarial}    & \multirow{3}{*}{Reward}      & halfcheetah & 20$\%$  & 3.0 & 5.99 $\times 10^5$ & 4.0 & 0.7 \\
    &  &  walker2d & 20$\%$  & 3.0 & 1.81 $\times 10^5$ & 4.0 &  0.5\\
    & & hopper & 10$\%$ & 5.0 &  2.01$\times 10^5$ & 6.0 & 0.7 \\
    \cline{2-8}
     & \multirow{3}{*}{Dynamics}      & halfcheetah & 30$\%$ & 1.2 & 3.60 $\times 10^5$ & 4.0 & 0.2  \\
    &  &  walker2d & 10$\%$ & 0.3 & 9.05 $\times 10^3$ & 4.0 & 0.5 \\
    & & hopper & 10$\%$ & 0.5 & 2.01$\times 10^4$ & 6.0 & 1.0 \\
     \bottomrule
  \end{tabular}
  \end{adjustbox}
\end{table}

\section{Comparison with Uncertainty-weighted Actor Critic (UWAC)}
Our practical implementation algorithm, UWMSG, shares some similarities with UWAC in terms of utilizing uncertainty weighting technique for offline RL. However, there are three key differences between our approaches: 
\begin{itemize}
    \item [1.] We focus on offline RL with data corruption, rather than the general offline RL setting explored by UWAC. Therefore, in our setting, the uncertainty arises from both corrupted datasets and OOD actions.
    \item [2.] While UWAC penalizes OOD actions in the Q objective through minimizing $\mathbb{E}_{(x,a,r,x')\sim \mathcal{D}, a'\sim \pi_{\theta}(\cdot|x')} \Big[ \frac{\big(\widehat{\mathcal{T}} Q(x,a) - Q(x,a)\big)^2}{Var[Q(x',a')]} \Big]$, with the aim of reducing the importance of OOD actions, our uncertainty weighting focuses on penalizing in-dataset $(x,a)$ pairs.
    \item [3.] Another distinction lies in the uncertainty estimation methods employed. UWAC uses dropout uncertainty, while we utilize bootstrapped uncertainty in our work. However, it is worth noting that our approach is not limited to a specific type of uncertainty estimation. In the future, more advanced uncertainty estimation methods (e.g., \citep{sun2022daux,deng2023uncertainty}) can be applied to potentially enhance the performance of UWMSG.
\end{itemize}

\section{Additional Results}
\label{ap:additionl_exp}
\paragraph{Learning Curves}All learning curves are shown in Figure \ref{fig:experiments_halfcheetah}, Figure \ref{fig:experiments_walker2d}, and Figure \ref{fig:experiments_hopper}. We can find that (1) current offline RL methods are susceptible to data corruption, e.g., MSG, EDAC, SAC-N achieve poor and unstable performance under adversarial attacks, and (2) our proposed UWMSG method significantly improves performance under different data corruption scenarios. Moreover, we posit that the reason for the observed initial increase and subsequent significant decrease of EDAC and SAC-N performance in some cases may be attributed to the characteristics of Temporal Difference (TD) learning. Specifically, the effect of corruption needs to accumulate over time, which may necessitate an extended training period to destroy the performance. In contrast, our algorithm UWMSG does not suffer from this problem and exhibits stable performance.

\begin{figure}[ht]
    \centering
    \subfigure[]{\includegraphics[width=0.255\linewidth, trim=0 0 3 0, clip]{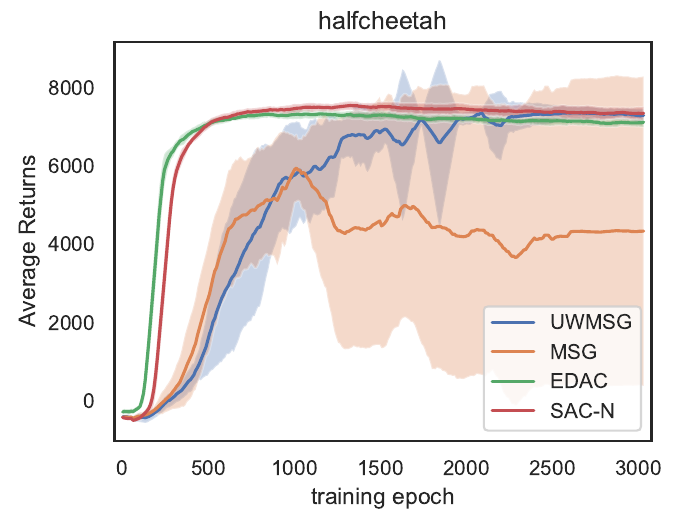}}
    \subfigure[]{\includegraphics[width=0.24\linewidth, trim=20 0 3 0, clip]{figs/halfcheetah_dynamics.pdf}}
     \subfigure[]{\includegraphics[width=0.24\linewidth, trim=20 0 3 0, clip]{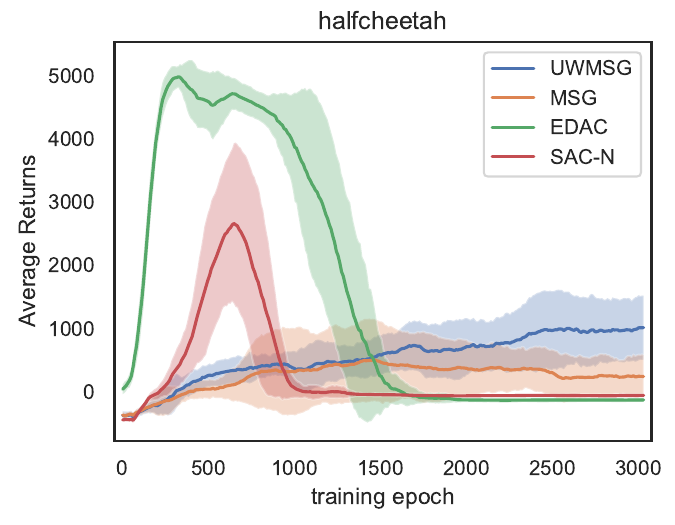}}
     \subfigure[]{\includegraphics[width=0.24\linewidth, trim=20 0 3 0, clip]{figs/halfcheetah_dynamics_adv.pdf}}
    \caption{Comparison on the halfcheetah task under (a) random reward, (b) random dynamics, (c) adversarial reward, and (d) adversarial dynamics attacks.}
    \label{fig:experiments_halfcheetah}
\end{figure}

\begin{figure}[ht]
    \centering
    \subfigure[]{\includegraphics[width=0.255\linewidth, trim=0 0 3 0, clip]{figs/walker2d_reward.pdf}}
    \subfigure[]{\includegraphics[width=0.24\linewidth, trim=20 0 3 0, clip]{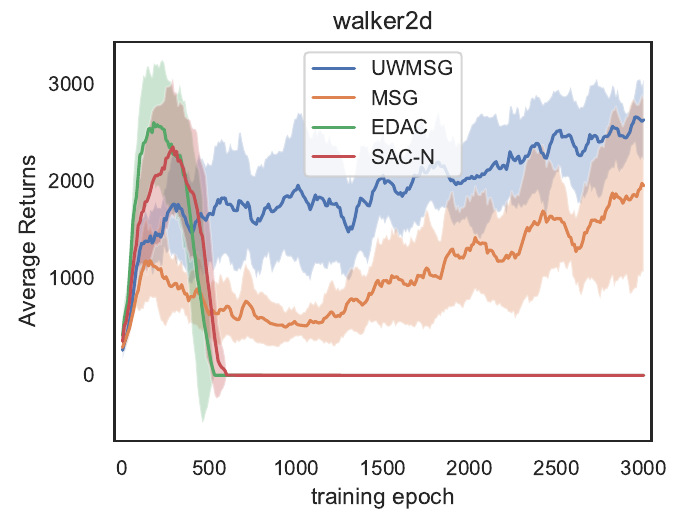}}
     \subfigure[]{\includegraphics[width=0.24\linewidth, trim=20 0 3 0, clip]{figs/walker2d_reward_adv.pdf}}
     \subfigure[]{\includegraphics[width=0.24\linewidth, trim=20 0 3 0, clip]{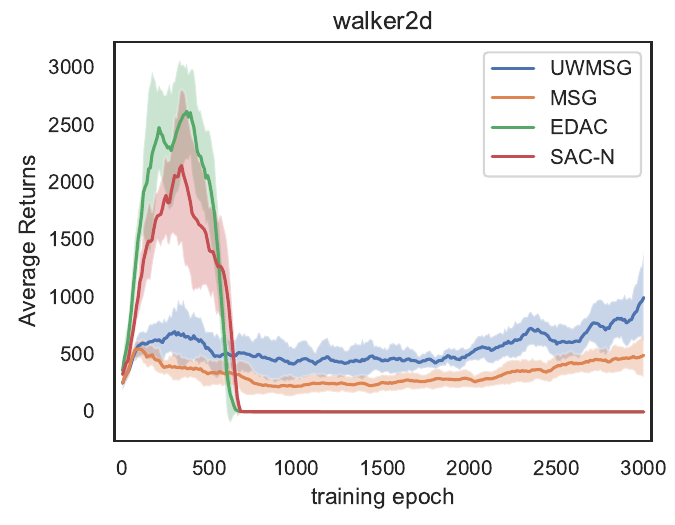}}
    \caption{Comparison on the walker2d task under (a) random reward, (b) random dynamics, (c) adversarial reward, and (d) adversarial dynamics attacks.}
    \label{fig:experiments_walker2d}
\end{figure}

\begin{figure}[ht]
    \centering
    \subfigure[]{\includegraphics[width=0.255\linewidth, trim=0 0 3 0, clip]{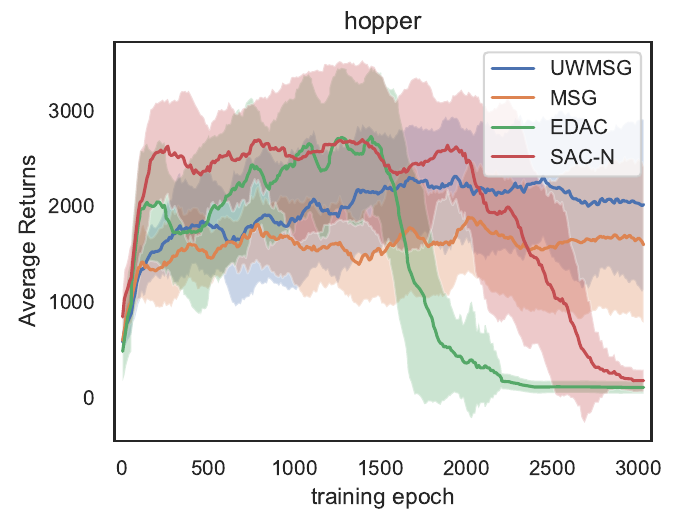}}
    \subfigure[]{\includegraphics[width=0.24\linewidth, trim=20 0 3 0, clip]{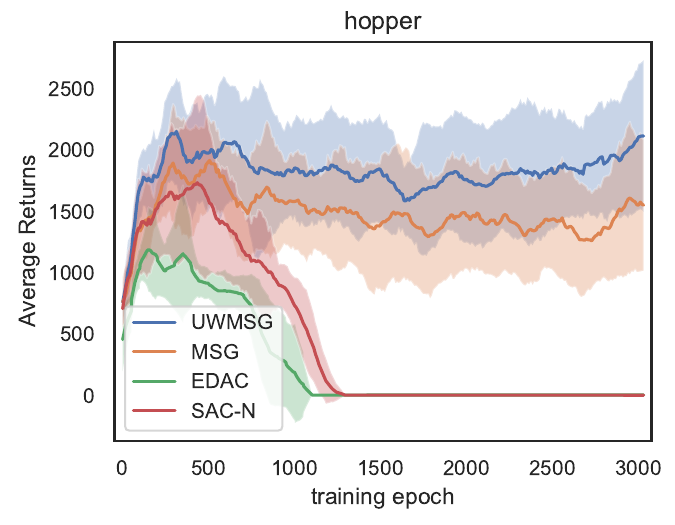}}
     \subfigure[]{\includegraphics[width=0.24\linewidth, trim=20 0 3 0, clip]{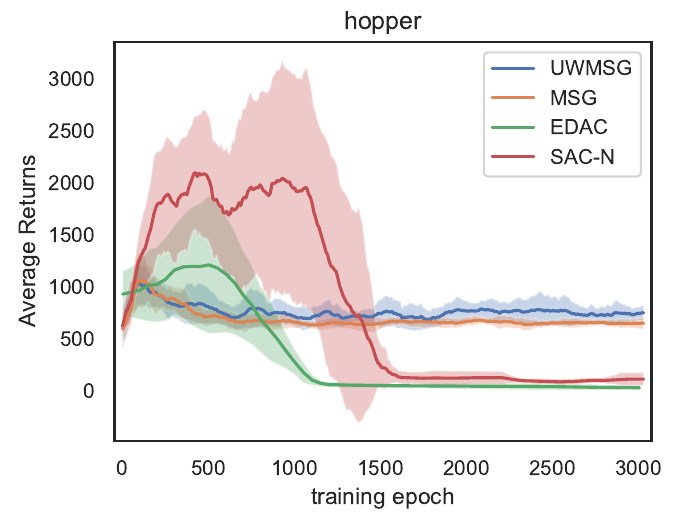}}
     \subfigure[]{\includegraphics[width=0.24\linewidth, trim=20 0 3 0, clip]{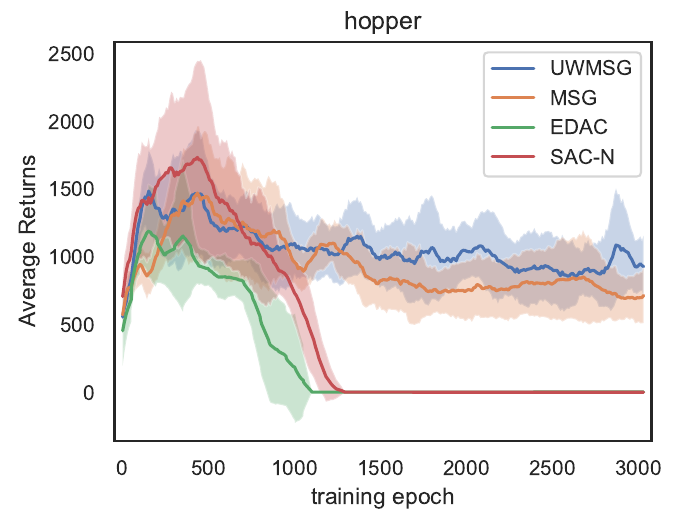}}
    \caption{Comparison on the hopper task under (a) random reward, (b) random dynamics, (c) adversarial reward, and (d) adversarial dynamics attacks.}
    \label{fig:experiments_hopper}
\end{figure}

\paragraph{Varying Corruption Level} We evaluate the performance of UWMSG under varying levels of corruption in Figure \ref{fig:corruption_level}. This is achieved by maintaining a consistent corruption scale in Table \ref{tab:hyper-attack} while adjusting the corruption rate. As depicted in the figure, as the cumulative corruption level rises, the overall performance of UWMSG progressively declines. These findings align with our theoretical analysis. Besides, the results indicate that dynamics corruption poses a greater challenge compared to reward corruption, leading to a larger drop in performance with a smaller corruption level.

\begin{figure}
    \centering
    \includegraphics[width=0.7\linewidth]{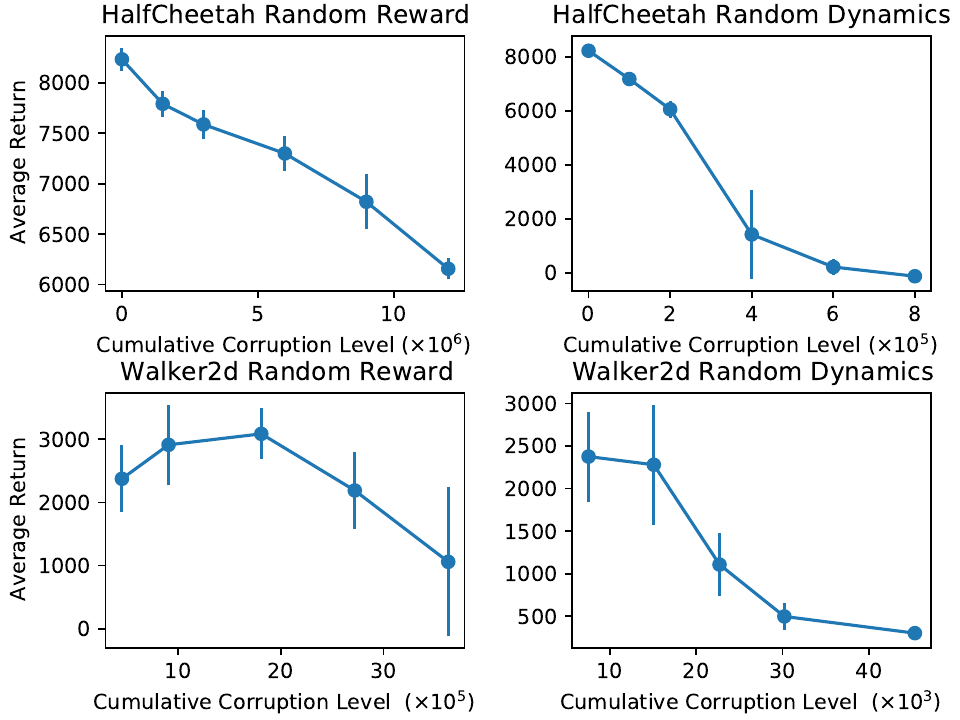}
    \caption{Performance of UWMSG under varying levels of corruption. Results are averaged over 5 random seeds.}
    \label{fig:corruption_level}
\end{figure}

\section{Technical Lemmas}
\begin{lemma}[Azuma–Hoeffding inequality, \citealt{cesa2006prediction}]\label{lemma:azuma}
Let $\{x_i\}_{i=1}^n$ be a martingale difference sequence with respect to a filtration $\{\cG_{i}\}$ satisfying $|x_i| \leq M$ for some constant $M$, $x_i$ is $\cG_{i+1}$-measurable, $\E[x_i|\cG_i] = 0$. Then for any $0<\delta<1$, with probability at least $1-\delta$, we have 
\begin{align}
    \sum_{i=1}^n x_i\leq M\sqrt{2n \log (1/\delta)}.\notag
\end{align} 
\end{lemma}

\begin{lemma}[Lemma 3.1 of \citet{jin2021pessimism}]\label{lm:Suboptimality Decomposition}
Let $\hat\pi=\{\hat\pi^h\}_{h=1}^H$ be the greedy policy such that for any $x$, $\hat\pi^h(x)=\argmax_{a\in\cA}f_n^h(x,a)$. For any initial state $x\in\cX$,
\$
\subopt(\hat\pi,x) &= \sum_{h=1}^H\E_{\pi_*}\big[f_n^h(x^h,\pi_*(x^h)) - f_n^h(x^h,\hat{\pi}(x^h)) \,\big|\, x^1=x\big]\notag\\
&\qquad - \sum_{h=1}^H\E_{\pi_*}\big[ \cE^h(f_n,x^h,a^h)\,\big|\, x^1=x\big] + \sum_{h=1}^H\E_{\hat{\pi}}\big[ \cE^h(f_n,x^h,a^h)\,\big|\, x^1=x\big],
\$
where $\cE^h(f,x^h,a^h)=f^h(x^h,a^h)-(\cT^h f^{h+1})(x^h,a^h)$ is the Bellman residual.
\end{lemma}

\begin{lemma}\label{lm:epsilon}
Let $\{\epsilon_s\}$ be a sequence of zero-mean conditional $\eta$-sub-Gaussian random variables: $\ln\E[e^{\lambda\epsilon_s}|\cS_{s-1}]\le\lambda^2\eta^2/2$, where $\cS_{s-1}$ represents the history data. We have for $t\ge1$, with probability at least $1-\delta$,
$$
\sum_{s=1}^t \epsilon_i^2\le 2t\sigma^2 + 3\sigma^2\ln(1/\delta).
$$
\end{lemma}
\begin{proof}
The proof can is presented in Lemma G.2 of \citet{ye2022corruptionrobust}.
\end{proof}

\begin{lemma}[Lemma G.4 of \citet{ye2022corruptionrobust}]\label{lm:empirical_diff_mdp}
Consider a function space $\cF:\cZ\rightarrow\rR$
and filtered sequence $\{z_t,\epsilon_t\}$ in $\cX\times\rR$ so that $\epsilon_t$ is conditional zero-mean $\eta$-sub-Gaussian noise. For $f_*(\cdot):\cZ\rightarrow\rR$, suppose that $y_t=f_*(z_t)+\epsilon_t$ and there exists a function $f_b\in\cF$ such that for any $t\in[T]$,
$
\sum_{s=1}^t|f_*(z_s)-f_b(z_s)| := \sum_{s=1}^t\zeta_s \le \zeta$. If $\hat{f}_t$ is an (approximate) ERM solution for some $\epsilon'\ge0$:
$$
\left(\sum_{s=1}^t(\hat{f}_t(z_s)-y_s)^2/\sigma_s^2\right)^{1/2} \le \min_{f\in\cF_{t-1}} \left(\sum_{s=1}^t(f(z_s)-y_s)^2/\sigma_s^2\right)^{1/2} + \sqrt{t}\epsilon',
$$
with probability at least $1-\delta$, we have for all $t\in[T]$:
\$
\sum_{s=1}^t(\hat{f}_t(z_s)-f_b(z_s))^2/\sigma_s^2 \le& 10\eta^2\ln(2N(\gamma,\cF,\|\cdot\|_{\infty})/\delta) + 5\sum_{s=1}^t|\hat{f_t}(z_s)-f_b(z_s)|\zeta_s/\sigma_s^2\\
&\qquad + 10(\gamma+\epsilon')\big((\gamma+\epsilon')t + \sqrt{tC_1(t,\zeta)}\big),
\$
where $C_1(t,\zeta)=2(\zeta^2 + 2t\eta^2 + 3\eta^2\ln(2/\delta))$.
\end{lemma}
\begin{proof}
The proof can be seen in Lemma G.4 of \citet{ye2022corruptionrobust}.
\end{proof}

\begin{lemma}[Theorem 1.3 of \citet{tropp2012user}]\label{lm:Matrix Hoeffding Concentration}
For a finite sequence $\{X_i\}_{i\in[n]}$ of independent, random and self-adjoint matrices with dimension $d$, let $\{A_i\}_{i\in[n]}$ be a sequence of fixed self-adjoint matrices. If each random matrix satisfies
\$
\E X_i = 0 \quad X_i^2 \preceq A_i^2 \quad \text{almost~surely},
\$
then, for all $t\ge0$,
\$
\Pb\Big(\lambda_{\max}\big(\sum_{i=1}^n X_i\big)\ge t\Big) \le d\cdot e^{-t^2/(8\sigma^2)},
\$
where $\sigma^2=\|\sum_{i=1}^n A_i^2\|_\op$.
\end{lemma}

\end{document}